\documentclass{article}
\PassOptionsToPackage{numbers,compress}{natbib}
\usepackage[preprint]{neurips_2026}

\usepackage[utf8]{inputenc} 
\usepackage[T1]{fontenc}    
\usepackage{hyperref}       
\usepackage{url}            
\usepackage{booktabs}       
\usepackage{amsfonts}       
\usepackage{nicefrac}       
\usepackage{microtype}      
\usepackage[dvipsnames,table]{xcolor}         

\usepackage{subcaption}
\usepackage{algorithm}
\usepackage{algpseudocodex} 
\usepackage{amsthm}
\usepackage{mathtools} 
\usepackage{enumitem}
\usepackage{multirow}
\usepackage{wrapfig}

\usepackage{comment}
\newcommand{\method}{{\textsc{DiSPO}}}
\usepackage[dvipsnames,table]{xcolor}

\newcommand{\action}[1]{\textbf{\textcolor{JungleGreen}{#1}}}
\newcommand{\state}[1]{\textbf{\textcolor{Orange}{#1}}}

\usepackage{tikz}
\usetikzlibrary{tikzmark}
\usepackage[most]{tcolorbox}
\usepackage{url}

\usepackage{multirow} 

\newtcbox{\roundboxgray}[1][gray!18]{
  on line,
  colback=#1,
  colframe=#1,
  boxrule=0pt,
  arc=3pt,
  boxsep=2pt,
  left=0pt, right=0pt, top=0pt, bottom=0pt
}

\newtcbox{\roundboxblue}[1][blue!10]{
  on line,
  colback=#1,
  colframe=#1,
  boxrule=0pt,
  arc=3pt,
  boxsep=1pt,
  left=0pt, right=0pt, top=0pt, bottom=0pt
}

\newtcbox{\roundboxorange}[1][orange!10]{
  on line,
  colback=#1,
  colframe=#1,
  boxrule=0pt,
  arc=3pt,
  boxsep=1pt,
  left=0pt, right=0pt, top=0pt, bottom=0pt
}

\newtheorem{theorem}{Theorem}[section]

\newtheorem{proposition}[theorem]{Proposition}

\hypersetup{
  colorlinks=true,
  linkcolor=red,
  citecolor=teal,
  urlcolor=magenta
}
\usepackage{mdframed}

\surroundwithmdframed[
  backgroundcolor=gray!5,
  linecolor=gray!50,
  linewidth=0.5pt,
  innertopmargin=6pt,
  innerbottommargin=6pt
]{proposition}

\surroundwithmdframed[
  backgroundcolor=gray!5,
  linecolor=gray!50,
  linewidth=0.5pt,
  innertopmargin=6pt,
  innerbottommargin=6pt
]{corollary}

\surroundwithmdframed[
  backgroundcolor=gray!5,
  linecolor=gray!50,
  linewidth=0.5pt,
  innertopmargin=6pt,
  innerbottommargin=6pt
]{theorem}

\title{Diffusion-State Policy Optimization for \\Masked Diffusion Language Models}

\author{%
  Daisuke Oba\textsuperscript{1} \quad
  Hiroki Furuta \quad
  Naoaki Okazaki\textsuperscript{1}\\
  \textsuperscript{1}Institute of Science Tokyo\\
  \textsuperscript{1}\texttt{\{daisuke.oba@nlp.,okazaki@\}comp.isct.ac.jp}
}

\begin{document}

\maketitle

\begin{abstract}
Masked diffusion language models generate text through iterative masked-token filling, but terminal-only rewards on final completions provide coarse credit assignment for the intermediate filling decisions that shape the generation process.
We propose Diffusion-State Policy Optimization (\method), a plug-in credit-assignment layer that directly optimizes intermediate filling decisions.
At selected intermediate masked states, \method\ branches by resampling the currently masked positions from rollout-cached logits, scores the resulting completions, and updates only the newly filled tokens, requiring no additional multi-step diffusion rollouts or optimizer steps.
We formalize a fixed-state objective for branched completions and derive a policy-gradient estimator that reuses the same rollouts as terminal-feedback policy optimization.
Experiments on LLaDA-8B-Instruct show that \method\ consistently improves terminal-feedback baselines, including diffu-GRPO and SPG, on math and planning benchmarks under matched rollout compute and optimizer steps, supporting its use as a general plug-in for masked diffusion policy optimization.
\end{abstract}

\section{Introduction}\label{sec:intro}
Policy optimization (PO) is widely used to improve the reasoning and alignment behavior of language models, from preference-based RLHF~\citep{ziegler2019fine,stiennon2020summarize,ouyang2022training} to policy-gradient methods~\citep{williams1992simple,schulman2015trust,schulman2017proximal} and recent reasoning-oriented variants~\citep{shao2024deepseekmath}.
Most PO methods, however, learn from scalar rewards on final completions, effectively assigning one sequence-level outcome to many local generation decisions.
This coarse credit assignment motivates our central question: \emph{can PO exploit intermediate information already produced during generation, rather than only the final outcome?}

This question is particularly natural for masked diffusion language models (MDLMs), which generate text by repeatedly filling masked positions over multiple denoising steps~\citep{austin2021structured,sahoo2024simple,shi2024simplified,nie2025largelanguagediffusionmodels,ye2025dream,gong2025diffucoder}.
During rollout, each denoising step already produces token-level distributions over many masked positions, but terminal-feedback PO trains only from a scalar final-completion reward, leaving this intermediate state-action information largely unused for fine-grained credit assignment~\citep{zhao2025d1,spg}.
Prior work goes beyond purely terminal feedback, e.g., via reward shaping based on intermediate states or trajectory-level comparisons at fixed steps~\citep{sapo,mdpo}, densifying supervision over realized denoising trajectories.
In contrast, we make the state-conditioned filling action the unit of optimization: under the fixed intermediate masked sequence, we compare alternative assignments to the current masks.

To exploit intermediate rollout information for state-conditioned credit assignment, we propose \emph{Diffusion-State Policy Optimization} (\method), a plug-in objective for MDLM policy optimization (Fig.~\ref{fig:concept}).
Motivated by policy-gradient learning~\citep{williams1992simple,sutton1999policy}, \method\ treats an intermediate masked sequence as the \state{state} and its mask fillings as the \action{action}.
At selected states, it fixes the masked sequence, samples alternative fillings from rollout-cached logits, scores the resulting branched completions with the base terminal reward function, and updates only the newly filled tokens.
Since \method\ only requires intermediate masked states and masked-position logits already produced during rollout, it adds same-state training comparisons without additional multi-step diffusion rollouts or optimizer updates.
Thus, \method\ can be plugged into terminal-feedback MDLM PO methods as a credit-assignment objective (Alg.~\ref{alg:dispo}).

We formalize \method\ as a fixed-state expected-return objective over branched completions and derive a policy-gradient estimator for state-conditioned mask filling (Theorem~\ref{thm:step-pg}).
This objective can be mixed with terminal-feedback PO using the same rollouts, giving a joint objective in expectation and supporting \method\ as a plug-in objective (Theorem~\ref{thm:mixed}).
Our analysis further motivates updating only newly filled tokens and averaging over same-state branches, which reduce variance and improve the efficiency of the step-wise update (Propositions~\ref{prop:var} and~\ref{prop:compute}).

We evaluate \method\ on LLaDA-8B-Instruct~\citep{nie2025largelanguagediffusionmodels} by integrating it with two terminal-feedback MDLM policy-optimization methods, diffu-GRPO~\citep{zhao2025d1} and SPG~\citep{spg}.
Under matched dominant training costs---multi-step diffusion rollouts and optimizer updates---\method\ consistently improves both base optimizers on math reasoning (GSM8K~\citep{gsm8k}, MATH500~\citep{lightman2023lets}) and symbolic planning (Sudoku, Countdown~\citep{zhao2025d1}).
Ablations attribute these gains to same-state branching and newly filled-token updates, confirming that the improvements come from the state-conditioned credit assignment (\S~\ref{sec:ablation}).

\begin{figure}[t]
    \centering
    \includegraphics[width=0.99\textwidth]{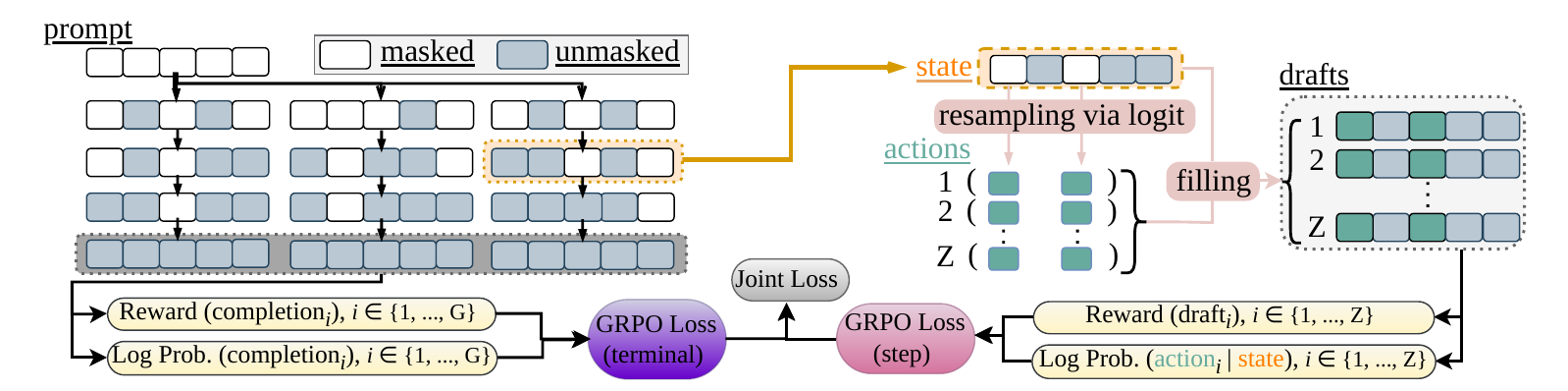}
    \caption{\textbf{Conceptual overview.}
    \textbf{\emph{Left}}: Terminal-feedback PO assigns reward to the final denoising trajectory.
    \textbf{\emph{Right}}: DiSPO branches from \state{intermediate masked states} using cached logits, scores the branched completions with the same reward, and updates only \action{newly filled tokens.}}
    \label{fig:concept}
\end{figure}

\section{Preliminaries}\label{sec:background}

\subsection{Terminal-feedback policy optimization for LMs}
\label{sec:bg:grpo}
We view a language model as a conditional policy $\pi_\theta(o\mid q)$ over completions $o$ given a prompt $q$.
In \emph{terminal-feedback} policy optimization, a scalar reward $R(q,o)$ is evaluated on the \emph{final} completion, and learning updates are driven by likelihood-ratio (policy-gradient) objectives at the \emph{sequence level}.
This bandit-style view is broadly applicable, but it attributes the same terminal signal to all decoding decisions, yielding coarse credit assignment.

A common instantiation is \emph{group-based} policy optimization~\citep{shao2024deepseekmath}.
Given $G$ sampled completions $o_1,\dots,o_G \sim \pi_{\theta_{\text{old}}}(\cdot\mid q)$ and terminal rewards $r_i = R(q,o_i)$, define the group baseline
$\bar r = \frac{1}{G}\sum_{i=1}^G r_i$
and advantages $A_i = r_i - \bar r$.
A concise likelihood-ratio objective (optionally with a KL penalty to a reference policy $\pi_{\text{ref}}$) is
\[
\mathbb{E}_{q}
\left[
\frac{1}{G}\sum_{i=1}^G
\underbrace{\frac{\pi_\theta(o_i\mid q)}{\pi_{\theta_{\text{old}}}(o_i\mid q)}}_{\rho_i(\theta)}
\,A_i
\!-\!\beta\,\mathrm{KL}\!\left(\pi_\theta(\cdot\mid q)\,\|\,\pi_{\text{ref}}(\cdot\mid q)\right)\right].
\]
We use this terminal-feedback formulation as a convenient baseline and as a plug-in component for our method.

\subsection{Surrogate likelihood for terminal-feedback optimization of MDLMs}
\label{sec:bg:diffu}
Masked diffusion language models (MDLMs) generate text by iteratively refining a partially-masked sequence, unmasking tokens in an arbitrary order.
This makes the exact sequence likelihood $\log \pi_\theta(o\!\mid\! q)$ intractable, which prevents directly applying standard terminal-feedback objectives such as \S~\ref{sec:bg:grpo}.
To enable \emph{terminal-feedback} policy optimization for MDLMs, \citet{zhao2025d1} introduces a tractable \emph{surrogate likelihood} $\tilde{\pi}_\theta$ (yielding an MDLM instantiation often referred to as diffu-GRPO).

Given a prompt--completion pair $(q,o)$, we sample a random prompt-masking pattern $m$ to obtain a corrupted prompt $q^{(m)}$.
We then form a fully masked completion $\tilde{o}$, run a single denoising step to obtain per-position token distributions
$p_\theta(\cdot \mid q^{(m)}, \tilde{o})$, and define the surrogate log-probability
\begin{equation}
  \log \tilde{\pi}_\theta(o\mid q)
\;=\;
\mathbb{E}_{m}\!\left[\;\sum_{j=1}^{|o|}
\log p_\theta\!\bigl(o_j \mid q^{(m)}, \tilde{o}\bigr)\;\right].
  \label{eq:global-surrogate}
\end{equation}
In terminal-feedback MDLM policy optimization, $\tilde{\pi}_\theta$ replaces $\pi_\theta$ inside likelihood ratios and KL terms, allowing MDLMs to be trained as conditional policies from scalar rewards on final completions.

\subsection{Rethinking MDLM decision points for RL}
\label{sec:bg:pg}
In an episodic MDP, the policy gradient theorem~\citep{williams1992simple,sutton1999policy} states that for a stochastic policy $\pi_\theta(a_t\mid s_t)$,
\[
\nabla_\theta J(\theta)
=
\mathbb{E}\!\left[\sum_t \nabla_\theta \log \pi_\theta(a_t\mid s_t)\, G_t\right],
\]
where $s_t$ and $a_t$ are the state and action, and $G_t$ denotes the return from step $t$.
Each decision $(s_t,a_t)$ contributes an explicit term
$\nabla_\theta \log \pi_\theta(a_t\!\mid\! s_t)\,G_t$.
This decomposition enables fine-grained credit assignment across steps.

Most RL formulations for LMs (autoregressive or diffusion) collapse decoding to a \emph{terminal-reward bandit}, where reward evaluation and gradient attribution are both driven by the final completion~\citep{shao2024deepseekmath,zhao2025d1}.
MDLMs, however, provide an anytime completion interface: at an intermediate masked state, a single forward pass yields logits for \emph{all} masked positions.
Hence we can view {the masked sequence} as the {state} $s_t$ and {its mask fillings} as the {action} $a_t$; filling deterministically yields a completion to score with the same terminal reward, and cached logits enable multiple same-state counterfactual completions without additional rollout forward passes.
This state--action view motivates our diffusion-state policy optimization in \S~\ref{sec:method}.

\section{Diffusion-State Policy Optimization}
\label{sec:method}
We present \emph{\textbf{Diffusion-State Policy Optimization}} (\textbf{\method}), a state-wise credit-assignment objective for MDLM policy optimization (Algo.~\ref{alg:dispo}).
\method\ branches at selected intermediate masked states using rollout-cached logits, scores the branched completions with the base terminal reward, and updates only newly filled tokens.
By reusing intermediate states and masked-position logits from rollout, \method\ adds same-state comparisons without additional multi-step diffusion rollouts and optimizer steps.
\emph{Importantly}, \method\ defines only the state-wise objective and treats the terminal-feedback PO objective as a replaceable base optimizer; for clarity, we first describe \method\ independently of the terminal objective and then instantiate the combined with diffu-GRPO~\citep{zhao2025d1}.

\subsection{RL Formulation}
\label{sec:method-states}

\textbf{{States.}}
For a prompt $q \!\sim\! \mathcal{D}$ and behavior parameters $\theta_{\text{old}}$, a multi-step MDLM sampler produces a trajectory of partially-masked sequences
$x_{k,1},x_{k,2},\dots,x_{k,T}$, where $T$ is the number of denoising steps and $k$ indexes trajectories.
Each $x_{k,t}\!\in\!(\mathcal{V}\!\cup\!\{\texttt{[MASK]}\})^{L}$ is a length-$L$ sequence with a subset of positions masked.
Let
\[
  M_{k,t}=\{i:\, x_{k,t,i}=\texttt{[MASK]}\},\quad
  U_{k,t}=[L]\setminus M_{k,t},
\]
and define {the diffusion state} at step $t$ as $s_{k,t}=(q,x_{k,t})$.

\textbf{{Actions} and policy factorization.}
At state $s_{k,t}$, the MDLM outputs logits for all $i\in M_{k,t}$, inducing per-position token distributions $\pi_\theta(\cdot\!\mid\! s_{k,t},i)$.
We define the {action} as the joint filling of the currently masked positions,
$a_{k,t}=\{a_{k,t,i}\}_{i\in M_{k,t}}$, with the factorized policy
\begin{equation}
  \pi_\theta(a_{k,t}\mid s_{k,t})
  \;=\;
  \prod_{i\in M_{k,t}}
  \pi_\theta\bigl(a_{k,t,i}\mid s_{k,t},i\bigr).
  \label{eq:one-step-policy}
\end{equation}
Applying an action yields a completion deterministically:
$o=\textsc{Fill}(x_{k,t},a_{k,t})$, which replaces $x_{k,t,i}$ by $a_{k,t,i}$ for $i\in M_{k,t}$ and keeps $U_{k,t}$ fixed.
This state--action view makes it natural to evaluate intermediate decisions using the same terminal reward $R(q,o)$.
Moreover, DiSPO attributes gradients only to the action distribution $\pi_\theta(a_{k,t}\mid s_{k,t})$, i.e., the newly filled tokens on $M_{k,t}$, treating previously filled tokens as part of the state.
In practice, likelihood ratios and KL terms are computed with a tractable masked-token surrogate introduced in \S~\ref{sec:method-surrogate}.

\subsection{Per-state Branching and Rewards}
\label{sec:method-branch}

For each prompt $q\!\sim\!\mathcal{D}$, we first sample $K$ behavior trajectories by running the multi-step MDLM sampler with frozen parameters $\theta_{\text{old}}$,
obtaining $\{x_{k,1:T}\}_{k=1}^K$ and terminal completions $o_k\!=\!x_{k,T}$.
During rollout, we cache the logits for each masked position $i\!\in\! M_{k,t}$ at every visited state $s_{k,t}$.

We then select a set of state indices $\mathcal{S}\subseteq\{(k,t)\}$ to update.
For each selected $(k,t)\in\mathcal{S}$, we form a \emph{state-conditioned} group by branching from the state:
using cached logits at $s_{k,t}$, we resample $Z$ actions
\[
  a_{k,t,1},\dots,a_{k,t,Z}
  \!\sim\!
  \pi_{\theta_{\text{old}}}(\cdot\!\mid\! s_{k,t})
  \!=\! \prod_{i\in M_{k,t}} \!\pi_{\theta_{\text{old}}}(\cdot\!\mid\! s_{k,t},i),
\]
and obtain completions $o_{k,t,z}\!=\!\textsc{Fill}(x_{k,t},a_{k,t,z})$.
Since branching reuses rollout-cached logits, generating these $Z$ completions incurs no additional \emph{rollout} forward passes.

Each completion is scored by the same terminal reward function,
\[
  R_{k,t,z} = R(q,o_{k,t,z}),
\]
and we compute a per-state baseline and group-relative advantages:
\begin{equation}
  \bar R_{k,t} = \frac{1}{Z}\sum_{z=1}^Z R_{k,t,z},
  \quad
  A_{k,t,z} = R_{k,t,z}-\bar R_{k,t}.
  \label{eq:group-adv}
\end{equation}
These advantages capture state-conditioned preferences over alternative fillings under a fixed intermediate state.

\subsection{Step-level Surrogate Log-probabilities}
\label{sec:method-surrogate}
To optimize masked-token actions at intermediate states, we need tractable log-probabilities to form likelihood ratios for the diffusion-state policy.
Following \citet{zhao2025d1}, we use a one-step masked-token surrogate $\tilde{\pi}_\theta$ (Eq.~\ref{eq:global-surrogate}), and adapt it to a \emph{state-wise definition} that accounts only for the currently masked positions.

\textbf{State-wise masked-token surrogate.}
For each state $s_{k,t}=(q,x_{k,t})$, sample a random prompt-masking pattern $m$ to obtain a corrupted prompt $q^{(m)}$ and define
$s_{k,t}^{(m)}\!=\!(q^{(m)},x_{k,t})$.
Running a single denoising step on $s_{k,t}^{(m)}$ yields per-position distributions $\tilde{\pi}_\theta(\cdot\!\mid\! s_{k,t}^{(m)},i)$ for all $i\in M_{k,t}$.
For an action sample $a_{k,t,z}$, define the state-wise surrogate log-probability by summing over actionable positions:
\begin{equation}
  \log \tilde{\pi}_\theta(a_{k,t,z}\!\mid\! s_{k,t})
  \!:=\!
  \mathbb{E}_{m}\!\!\left[
    \sum_{i\in M_{k,t}}
      \!\!\log \tilde{\pi}_\theta\bigl(a_{k,t,z,i}\!\mid\! s_{k,t}^{(m)}\!,\!i\bigr)
  \right].
  \label{eq:subset-logprob-method}
\end{equation}
This restriction makes the update token-local: rewards computed from completions affect only the newly filled tokens at the current state.

\textbf{Implementation note.}
We estimate the expectation over $m$ with the same Monte Carlo scheme as \citet{zhao2025d1}.
All log-probabilities, likelihood ratios, and KL terms in our objectives are computed using $\tilde{\pi}_\theta$.

\subsection{Training Objective}
\label{sec:method-objective}
For clarity, we present the \emph{unclipped} likelihood-ratio objectives; our implementation uses standard clipping and a KL penalty to a reference surrogate policy $\tilde{\pi}_{\text{ref}}$.

\textbf{State-wise DiSPO objective.}
For each selected state $(k,t)\in\mathcal{S}$, we have $Z$ branched actions with rewards $\{(a_{k,t,z},R_{k,t,z})\}_{z=1}^Z$ and group-relative advantages from Eq.~\ref{eq:group-adv}.
Using the state-wise surrogate in Eq.~\ref{eq:subset-logprob-method}, define
\[
  \rho_{k,t,z}(\theta)
  =
  \exp\!\Bigl(
    \log \tilde{\pi}_\theta(a_{k,t,z}\mid s_{k,t})
    -
    \log \tilde{\pi}_{\theta_{\mathrm{old}}}(a_{k,t,z}\mid s_{k,t})
  \Bigr).
\]
The state-wise loss by \method\ for state $(k,t)$ is
\begin{equation}
  \mathcal{L}_{\text{step}}^{(k,t)}(\theta)
  =
  - \frac{1}{Z}\sum\nolimits_{z=1}^{Z} \rho_{k,t,z}(\theta)\, A_{k,t,z},
  \label{eq:inter-loss-method}
\end{equation}
and the aggregated state-wise loss is
\begin{equation}
    \textcolor{Magenta}{\mathcal{L}_{\text{step}}(\theta)}=\sum\nolimits_{(k,t)\in\mathcal{S}} \mathcal{L}_{\text{step}}^{(k,t)}(\theta).
\end{equation}

\begin{algorithm}[t]
\small
\caption{\method\ with a Terminal-Feedback Base Optimizer}
\label{alg:dispo}
\begin{algorithmic}[1]
\Require initial parameters $\theta$, reward function $R$, dataset $\mathcal{D}$, trajectory size $K$, branch size $Z$, batch size $B$,
\Statex \hspace{\algorithmicindent} timestep sampler $\omega(t)$,
weights $\alpha_{\mathrm{step}}, \alpha_{\mathrm{base}}$,
objectives
$\textcolor{Violet}{\mathcal{L}_{\mathrm{base}}(\cdot)}$,
$\textcolor{Magenta}{\mathcal{L}_{\mathrm{step}}(\cdot)}$,
learning rate $\eta$.

\Repeat
  \State \textbf{initialize:} $\theta_{\mathrm{old}} \leftarrow \theta$, $\mathcal{L} \leftarrow 0$
  \State \textbf{sample prompts:} $\{q_b\}_{b=1}^B \sim \mathcal{D}$

  \For{each prompt $q_b$}

    \BeginBox[fill=Violet!5, draw=Violet!30, rounded corners]
      \State \textbf{rollout:} sample $K$ denoising trajectories with $\theta_{\mathrm{old}}$; cache logits at each visited state $s_{k,t}=(q_b,x_{k,t})$; obtain terminal completions $\{o_k\}_{k=1}^K$.

      \State \textbf{terminal rewards:}
      $\{R_k\}_{k=1}^K \leftarrow \{R(q_b,o_k)\}_{k=1}^K$

      \State \textbf{terminal loss:}
      $\begin{aligned}[t]
      \mathcal{L}
      \leftarrow\;&
      \mathcal{L}
      + \alpha_{\mathrm{base}}\;
      \textcolor{Violet}{
      \mathcal{L}_{\mathrm{base}}
      \bigl(
      \{(q_b,o_k,R_k)\}_{k=1}^K;
      \tilde{\pi}_\theta,
      \tilde{\pi}_{\theta_{\mathrm{old}}}
      \bigr)}
      \end{aligned}$
    \EndBox\BeginBox[fill=Magenta!4, draw=Magenta!30, rounded corners]
      \State \textbf{sample timesteps:}
      $\mathcal{T}_{\mathrm{sub}} \sim \omega(t)$ {\quad \footnotesize $\triangleright$ branching reuses cached logits; no additional rollout passes}

      \State \textbf{select states:}
      $\mathcal{S}_b \leftarrow \{1,\dots,K\}\times \mathcal{T}_{\mathrm{sub}}$

      \For{each selected $(k,t)\in \mathcal{S}_b$}
        \State \textbf{same-state branching:}
        sample actions $\{a_{k,t,z}\}_{z=1}^Z$ from cached logits at $s_{k,t}$

        \State \textbf{form completions:}
        $o_{k,t,z}=\textsc{Fill}(x_{k,t},a_{k,t,z})$

        \State \textbf{step rewards:}
        $\{R_{k,t,z}\}_{z=1}^Z
        \leftarrow
        \{R(q_b,o_{k,t,z})\}_{z=1}^Z$

        \State \textbf{pack outcomes:}
        $\mathcal{B}_{k,t}
        \leftarrow
        \{(a_{k,t,z},R_{k,t,z})\}_{z=1}^Z$

        \State \textbf{step-wise logps:}
        $\{\log\tilde{\pi}_\theta(a_{k,t,z}\mid s_{k,t})\}_{z=1}^Z$
        and
        $\{\log\tilde{\pi}_{\theta_{\mathrm{old}}}
        (a_{k,t,z}\mid s_{k,t})\}_{z=1}^Z$

        \State \textbf{step-wise loss:}
        $\begin{aligned}[t]
        \mathcal{L}
        \leftarrow\;&
        \mathcal{L}
        + \alpha_{\mathrm{step}}\;
        \textcolor{Magenta}{
        \mathcal{L}_{\mathrm{step}}
        \bigl(
        s_{k,t},
        \mathcal{B}_{k,t};
        \tilde{\pi}_\theta,
        \tilde{\pi}_{\theta_{\mathrm{old}}}
        \bigr)}
        \end{aligned}$
    \EndBox
      \EndFor
  \EndFor

  \State \textbf{update:}
  $\theta \leftarrow \theta - \eta \nabla_\theta \mathcal{L}$
  {\quad \footnotesize $\triangleright$ matched number of update steps}

\Until{convergence}

\end{algorithmic}
\end{algorithm}

\textbf{Combination with a base terminal objective.}
DiSPO is agnostic to the choice of terminal-feedback optimizer.
Let $\mathcal{L}_{\mathrm{base}}(\theta)$ denote a terminal-feedback MDLM PO objective computed from the same rollout completions and terminal reward function.
We optimize
\begin{equation}
  \mathcal{L}(\theta)
  =
  \alpha_{\mathrm{base}}\,\textcolor{Violet}{\mathcal{L}_{\mathrm{base}}(\theta)}
  +
  \alpha_{\mathrm{step}}\,\textcolor{Magenta}{\mathcal{L}_{\mathrm{step}}(\theta)},
  \label{eq:total-loss}
\end{equation}
where $\alpha_{\mathrm{base}},\alpha_{\mathrm{step}}\in\mathbb{R}_{\geq 0}$.
This separates the proposed state-wise credit-assignment objective from the replaceable terminal-feedback base optimizer. 
We set $\alpha_{\text{base}}\!=\!1$ in experiments.

\textbf{diffu-GRPO instantiation.}
For controlled comparisons with prior MDLM policy optimization, we instantiate $\mathcal{L}_{\mathrm{base}}$ with diffu-GRPO~\citep{zhao2025d1}.
For each prompt $q$, rollout completions $\{o_k\}_{k=1}^K$ are scored by $r_k=R(q,o_k)$ and converted to group-relative advantages
\(
  A_k = r_k-\bar r_q
\)
where $\bar r_q=\frac{1}{K}\sum_{k=1}^{K} r_k$.
Using the sequence-level surrogate log-probabilities $\log\tilde{\pi}_\theta(q,o_k)$ from Eq.~\ref{eq:global-surrogate}, define
\[
  \rho_k(\theta)
  =
  \exp\!\Bigl(
    \log\tilde{\pi}_\theta(q,o_k)
    -
    \log\tilde{\pi}_{\theta_{\mathrm{old}}}(q,o_k)
  \Bigr).
\]
The diffu-GRPO base loss is
\begin{equation}
  \textcolor{Violet}{\mathcal{L}_{\mathrm{base}}(\theta)}
  = -\frac{1}{K}\sum\nolimits_{k=1}^K \rho_k(\theta)\,A_k.
  \label{eq:diffu-grpo-base-loss}
\end{equation}
Other terminal-feedback objectives, such as SPG, can be used by replacing \textcolor{Violet}{$\mathcal{L}_{\mathrm{base}}$} in Eq.~\ref{eq:total-loss}.

\subsection{Theoretical Analysis}
\label{sec:moved:theory}

We summarize the theoretical properties of \method; formal statements and proofs are given in Appendix~\ref{sec:theory}.
The analysis considers the unclipped likelihood-ratio objective without the KL penalty.

\textbf{{Fixed-state policy gradient.}}
First, \method\ provides a valid policy-gradient signal for fixed-state mask filling.
For a fixed denoising step, conditioning on intermediate masked states induced by the frozen behavior policy, the expected gradient of the step-wise loss is aligned with the policy gradient of a fixed-state expected-return objective, up to the constant factor $c_Z=(Z-1)/Z$ induced by the within-state group baseline.
See Theorem~\ref{thm:step-pg} in Appendix~\ref{sec:theory-step}.

\textbf{{Mixed objective.}}
Second, the step-wise objective can be combined with a terminal-feedback base optimizer without introducing an ad-hoc learning signal.
In expectation, the combined loss optimizes a single mixed objective consisting of the base terminal objective and the timestep-averaged fixed-state objectives.
See Theorem~\ref{thm:mixed} in Appendix~\ref{sec:theory-mixed}.

\textbf{Variance-motivated design choices.}
Finally, we provide simplified variance analyses that motivate the two main design choices in \method: updating only newly filled tokens and averaging over same-state branches.
Under standard independence assumptions, token-local updates reduce variance by excluding non-actionable positions, while averaging $Z$ same-state branches yields the usual $O(1/Z)$ variance reduction.
See Propositions~\ref{prop:var} and~\ref{prop:compute} in Appendix~\ref{sec:theory-var}.

\section{Experiments}
\label{sec:experiments}
This section evaluates whether state-conditioned credit assignment improves terminal-feedback PO for MDLMs.
We plug \method\ into two base optimizers, diffu-GRPO~\citep{zhao2025d1} and SPG~\citep{spg}, by adding the state-wise objective in Algorithm~\ref{alg:dispo}.
For controlled comparisons, we match multi-step diffusion rollouts, sequence-level group size, and RL update steps.
\method\ forms branches only from logits cached in those rollouts, so gains reflect the added state-wise signal rather than extra rollout compute.

\subsection{Setup}
\label{sec:setup}

\textbf{Models.}
We evaluate LLaDA-8B-Instruct~\citep{nie2025largelanguagediffusionmodels} and its s1k-supervised fine-tuned variant~\citep{s1k}.
We follow \citet{zhao2025d1} for the model setup and denoising schedule. See Appendix~\ref{app:exp:setting} for details.

\textbf{Benchmarks and metrics.}
Following \citet{zhao2025d1}, we report exact-match accuracy on math reasoning benchmarks (GSM8K~\citep{gsm8k}; MATH500~\citep{lightman2023lets}) and symbolic planning benchmarks (Sudoku; Countdown~\citep{zhao2025d1}).
See Appendix~\ref{app:exp:benchmark} for details.

\textbf{Optimization.}
For each base optimizer, we match the dominant training budgets between the baseline and its \method-augmented variant, including $N_{\text{gen}}{=}256$, the denoising schedule (with block size 32), the rollout budget, and the number of RL update steps.
For SPG, we use the original paper's practical masking and mixed negative-advantage estimator settings, fixed for both SPG and \method$_\text{SPG}$.
We use the same terminal reward evaluator for both terminal and state-wise terms, without reward shaping.
\method\ uses $Z{=}2$ and samples one timestep per rollout from a late-biased polynomial distribution with degree $k{=}4$.
We sweep $\alpha_{\text{step}}\in\{0.1,0.5,0.9\}$.
See Appendix~\ref{app:exp:setting} for details.

\textbf{Evaluation.}
We follow the inference protocol of diffu-GRPO~\citep{zhao2025d1}.
We evaluate at $N_{\text{gen}}\in\{128,256,512\}$ using zero-shot greedy decoding.
For diffusion decoding, we use $N_{\text{gen}}/2$ denoising steps with a block size of 32 tokens.
Additional details are provided in Appendix~\ref{app:exp:inference}.

\begin{table*}[t]
\centering
\footnotesize
\caption{\textbf{\method{} on LLaDA.}
Exact-match accuracy (\%) on planning and math reasoning, evaluated with $N_{\text{gen}}\in\{128,256,512\}$.
Matched-compute comparisons are made within each base optimizer: diffu-GRPO vs.\ \method$_\text{diffu-GRPO}$ and SPG vs.\ \method$_\text{SPG}$, with the same training $N_{\text{gen}}{=}256$, multi-step diffusion rollout budget, and optimizer steps.
$\dagger$ indicates a non-matched-compute reward-shaping baseline included for reference.
$\ddagger$ indicates results reported by \citet{zhao2025d1}.}
\begin{tabular}{@{}l@{\,\,\,\,\,}r@{\,\,\,}r@{\,\,\,}r@{\,\,\,}rr@{\,\,\,}r@{\,\,\,}r@{\,\,\,}rr@{\,\,\,}r@{\,\,\,}r@{\,\,\,}rr@{\,\,\,}r@{\,\,\,}r@{\,\,\,}r@{}}
\toprule
& \multicolumn{4}{c}{\textbf{Sudoku}} & \multicolumn{4}{c}{\textbf{Countdown}} & \multicolumn{4}{c}{\textbf{GSM8K}} & \multicolumn{4}{c}{\textbf{MATH500}} \\
\cmidrule(lr){2-5}\cmidrule(lr){6-9}\cmidrule(lr){10-13}\cmidrule(lr){14-17}
\textbf{Methods} ~$\backslash$~ \textbf{$N_\text{gen}$} & \textbf{128} & \textbf{256} & \textbf{512} & \textbf{Avg.} & \textbf{128} & \textbf{256} & \textbf{512} & \textbf{Avg.} & \textbf{128} & \textbf{256} & \textbf{512} & \textbf{Avg.} & \textbf{128} & \textbf{256} & \textbf{512} & \textbf{Avg.} \\
\midrule
LLaDA-8B-Inst.~\cite{nie2025largelanguagediffusionmodels}$^\ddagger$ & 11.7 & 6.7 & 5.5 & 8.0 & 20.7 & 19.5 & 16.0 & 18.7 & 68.7 & 76.7 & 78.2 & 74.5 & 26.0 & 32.4 & 36.2 & 31.5 \\
\midrule
diffu-GRPO~\citep{zhao2025d1}$^\ddagger$ & 18.4 & 12.9 & 11.0 & 14.1 & 33.2 & 31.3 & 37.1 & 33.9 & 72.6 & 79.8 & {81.9} & 78.1 & 33.2 & 37.2 & 39.2 & 36.5 \\
\textbf{\method}$_\text{diffu-GRPO}$
  & \textbf{34.7} & \textbf{30.2} & \textbf{29.4} & \textbf{31.4}
  & \textbf{{49.6}} & \textbf{{49.6}} & \textbf{62.5} & \textbf{53.9}
  & \textbf{74.8} & \textbf{{81.9}} & \textbf{82.9} & \textbf{79.9}
  & \textbf{34.0} & \textbf{{39.2}} & \textbf{40.2} & \textbf{37.8} \\
\midrule
SPG~\citep{spg} & 26.6 & 26.0 & 26.1 & 26.2 & 67.6 & 66.8 & 67.2 & 67.2 & 77.3 & 80.9 & 81.5 & 79.9 & 33.6 & 38.4 & 40.2 & 37.4 \\
\textbf{\method}$_\text{SPG}$\ & \textbf{43.2} & \textbf{44.0} & \textbf{43.4} & \textbf{43.5} & \textbf{68.0} & \textbf{68.8} & \textbf{68.8} & \textbf{68.5} & \textbf{79.0} & \textbf{81.7} & \textbf{81.8} & \textbf{80.8} & \textbf{35.2} & \textbf{39.4} & \textbf{41.0} & \textbf{38.5}\\
\midrule
\textcolor{black!55}{SAPO~\citep{sapo}}$^\dagger$
  & \textcolor{black!55}{22.4} & \textcolor{black!55}{20.3} & \textcolor{black!55}{16.1} & \textcolor{black!55}{19.6}
  & \textcolor{black!55}{51.6} & \textcolor{black!55}{52.0} & \textcolor{black!55}{56.3} & \textcolor{black!55}{53.3}
  & \textcolor{black!55}{72.9} & \textcolor{black!55}{82.2} & \textcolor{black!55}{82.4} & \textcolor{black!55}{79.2}
  & \textcolor{black!55}{32.0} & \textcolor{black!55}{40.0} & \textcolor{black!55}{38.4} & \textcolor{black!55}{36.8} \\
\bottomrule
\end{tabular}
\label{tab:main-results}
\end{table*}

\begin{table*}[t]
\centering
\footnotesize
\caption{\textbf{\method{} on LLaDA-SFT.}
Under the same evaluation and matched-compute setup as Table~\ref{tab:main-results}, \method\ uses a conservative step weight ($\alpha_{\text{step}}{=}0.1$) and still improves SFT + diffu-GRPO across planning and math benchmarks. 
$\ddagger$ indicates results reported by \citet{zhao2025d1}.}
\begin{tabular}{@{}l@{\,\,\,\,\,}r@{\,\,\,}r@{\,\,\,}r@{\,\,\,}rr@{\,\,\,}r@{\,\,\,}r@{\,\,\,}rr@{\,\,\,}r@{\,\,\,}r@{\,\,\,}rr@{\,\,\,}r@{\,\,\,}r@{\,\,\,}r@{}}
\toprule
& \multicolumn{4}{c}{\textbf{Sudoku}} & \multicolumn{4}{c}{\textbf{Countdown}} & \multicolumn{4}{c}{\textbf{GSM8K}} & \multicolumn{4}{c}{\textbf{MATH500}} \\
\cmidrule(lr){2-5}\cmidrule(lr){6-9}\cmidrule(lr){10-13}\cmidrule(lr){14-17}
\textbf{Methods} ~$\backslash$~ \textbf{$N_\text{gen}$} & \textbf{128} & \textbf{256} & \textbf{512} & \textbf{Avg.} & \textbf{128} & \textbf{256} & \textbf{512} & \textbf{Avg.} & \textbf{128} & \textbf{256} & \textbf{512} & \textbf{Avg.} & \textbf{128} & \textbf{256} & \textbf{512} & \textbf{Avg.} \\
\midrule
LLaDA-8B-Inst.~\cite{nie2025largelanguagediffusionmodels}$^\ddagger$ & 11.7 & 6.7 & 5.5 & 8.0 & 20.7 & 19.5 & 16.0 & 18.7 & 68.7 & 76.7 & 78.2 & 74.5 & 26.0 & 32.4 & 36.2 & 31.5 \\
\midrule
SFT~\citep{s1k}$^\ddagger$ & 16.5 & 8.5 & 4.6 & 9.9 & 20.3 & 14.5 & 23.8 & 19.5 & 66.5 & 78.8 & 81.1 & 75.5 & 26.2 & 32.6 & 34.8 & 31.2 \\
SFT \& diffu-GRPO~\citep{zhao2025d1}$^\ddagger$ & 22.1 & 16.7 & 9.5 & 16.1 & 34.8 & 32.0 & \textbf{42.2} & 36.3 & 73.2 & 81.1 & 82.1 & 78.8 & {33.8} & {38.6} & \textbf{40.2} & {37.5} \\

SFT \& \textbf{\method}$_\text{diffu-GRPO}$ & \textbf{26.4} & \textbf{25.0} & \textbf{11.7} & \textbf{21.0} & \textbf{37.5} & \textbf{32.8} & 40.2 & \textbf{36.8} & \textbf{75.0} & \textbf{82.3} & \textbf{83.2} & \textbf{80.2} & \textbf{34.8} & \textbf{39.2} & \textbf{40.2} & \textbf{38.1} \\
\bottomrule
\end{tabular}
\label{tab:sft-results}
\end{table*}

\subsection{Main results}
\label{sec:main-results}

Table~\ref{tab:main-results} shows that \method\ improves both base optimizers, diffu-GRPO and SPG, on LLaDA-8B-Instruct across benchmarks.
The gains are largest on planning tasks and generally persist across evaluation budgets $N_{\text{gen}}\in\{128,256,512\}$.
Table~\ref{tab:sft-results} further shows that the same state-wise objective also improves after SFT, even with the conservative setting $\alpha_{\text{step}}{=}0.1$.

All matched comparisons use the same multi-step diffusion rollout budget, sequence-level group size, optimizer steps, and terminal reward evaluator.
Thus, the improvements are best explained by finer state-conditioned credit assignment rather than additional rollout compute or reward shaping.
Training-dynamics of rewards plots are provided at Figure~\ref{fig:qualitative-2x2} in Appendix~\ref{reward-curve}.

\begin{table}[t]
\centering
\footnotesize
\caption{\textbf{Ablation studies.}
(\textbf{Left}) Step-wise estimator design: gradient scope (Prop.~\ref{prop:var}) and number of same-state drafts $Z$ (Prop.~\ref{prop:compute}).
(\textbf{Right}) Optimization recipe: timestep sampling, loss components, and step-loss intensity $\alpha_{\text{step}}$.
All values represent the exact-match accuracy (\%) on Sudoku.}
\begin{minipage}[t]{0.48\linewidth}
\centering
\begin{tabular}{l@{\,\,}lllll}
\toprule
\multicolumn{2}{l}{\textbf{Perspectives} ~$\backslash$~ $N_\text{gen}$} & {\textbf{128}} & {\textbf{256}} & {\textbf{512}} & {\textbf{Avg.}} \\
\midrule
\multirow{2}{*}{\rotatebox[origin=c]{90}{\textbf{\scriptsize{Grad.}}}}
& $\triangleright$ All tokens & {22.7} & {21.4} & {\textbf{20.6}} & {21.6} \\
& $\triangleright$ Action only & \textbf{26.6}& \textbf{23.3} & 20.0 & \textbf{23.3} \\
\midrule
\multirow{3}{*}{\rotatebox[origin=c]{90}{\textbf{\scriptsize{\# Drafts}}}}
& $\triangleright$ $Z=2$ & {\textbf{26.6}} & {\textbf{23.3}} & {20.0} & {\textbf{23.3}} \\
& $\triangleright$ $Z=3$ & 25.1 & 22.8 & \textbf{21.0} & 23.0 \\
& $\triangleright$ $Z=4$ & \textbf{26.6} & 22.1 & 19.6 & 22.8 \\
\bottomrule
\end{tabular}
\subcaption{Step-wise estimator design.}
\label{tab:ablation-variance}
\end{minipage}
\hfill
\begin{minipage}[t]{0.48\linewidth}
\centering
\begin{tabular}{l@{\,\,}lllll}
\toprule
\multicolumn{2}{l}{\textbf{Perspectives} ~$\backslash$~ $N_\text{gen}$} & {\textbf{128}} & {\textbf{256}} & {\textbf{512}} & {\textbf{Avg.}} \\
\midrule
\multirow{3}{*}{\rotatebox[origin=c]{90}{\textbf{\scriptsize{Timestep}}}}
& $\triangleright$ Uniform & {25.1} & {23.0} & {17.5} & {21.9} \\
& $\triangleright$ Early-focused & {24.2} & {18.1} & {16.0} & {19.4} \\
& $\triangleright$ Late-focused & {\textbf{26.6}} & {\textbf{23.3}} & {\textbf{20.0}} & {\textbf{23.3}} \\
\midrule
\multirow{3}{*}{\rotatebox[origin=c]{90}{\textbf{\scriptsize{Loss}}}}
& $\triangleright$ $\mathcal{L}_{\mathrm{base}}$ only & {18.4} & {12.9} & {11.0} & {14.1} \\
& $\triangleright$ $\mathcal{L}_{\text{step}}$ only & {24.9} & {18.4} & {17.7} & {20.3} \\
& $\triangleright$ Both & \textbf{26.6} & \textbf{23.3} & \textbf{20.0} & \textbf{23.3} \\
\midrule
\multirow{4}{*}{\rotatebox[origin=c]{90}{\textbf{\scriptsize{Intensity}}}}
& $\triangleright$ $\alpha_{\text{step}}=0.0$ & {18.4} & {12.9} & {11.0} & {14.1} \\
& $\triangleright$ $\alpha_{\text{step}}=0.1$ & {26.6} & {23.3} & {20.0} & {23.3} \\
& $\triangleright$ $\alpha_{\text{step}}=0.5$ & \textbf{34.7} & \textbf{30.2} & \textbf{29.4} & \textbf{31.4} \\
& $\triangleright$ $\alpha_{\text{step}}=0.9$ & {28.8} & {27.8} & {26.3} & {27.6} \\
\bottomrule
\end{tabular}
\subcaption{Optimization recipe.}
\label{tab:ablation}
\end{minipage}
\end{table}

\subsection{Analysis: State-wise Estimator Design}
\label{sec:ablation-variance}
Propositions~\ref{prop:var} and~\ref{prop:compute} motivate two design choices for the state-wise estimator: i) updating only actionable tokens and ii) averaging over $Z$ same-state branches.
We test both on 4$\times$4 Sudoku through controlled gradient-variance measurements and task ablations.

\begin{wrapfigure}{r}{0.44\linewidth}
    \vspace{-1.5\baselineskip}
    \centering
    \includegraphics[width=\linewidth]{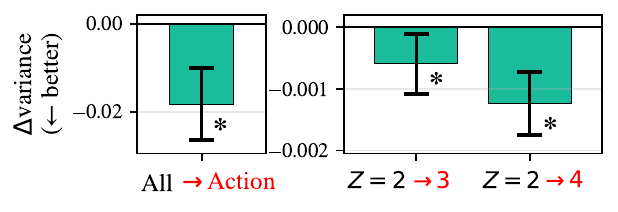}
    \caption{\textbf{Variance reduction of the step-wise gradients on Sudoku.}
        \textbf{\textit{Left}}: Updating only actions (vs.\ all tokens) reduces variance at $Z{=}2$. 
        \textbf{\textit{Right}}: Increasing $Z$ reduces variance. 
        Error bars: paired 95\% bootstrap CIs.
        }
  \label{fig:gradvar}
  \vspace{-1.1\baselineskip}
\end{wrapfigure}
\textbf{Gradient variance.}
Using fixed intermediate states from the final LLaDA-8B-Instruct Sudoku checkpoint, we resample branches and estimate the trace covariance of the state-wise gradient, $\mathrm{trCov}(\hat g)$.
Figure~\ref{fig:gradvar} confirms both predictions: i) action-only updates reduce variance relative to all-token updates, and ii) larger $Z$ reduces variance under action-only updates.

\textbf{Task performance.}
Table~\ref{tab:ablation-variance} shows that all-token updates underperform action-only updates, supporting the state--action decomposition.
Increasing $Z$ from 2 to 4 sometimes improves individual evaluation budgets but gives no monotonic average gain, so we use $Z{=}2$ by default.

\subsection{Ablations: Optimization Choices}
\label{sec:ablation}

\textbf{Timestep sampling.}
Table~\ref{tab:ablation} (top) compares uniform sampling with late- and early-biased polynomial samplers ($k{=}4$).
Late-biased sampling performs best, suggesting that states closer to the final completion provide more reliable credit-assignment signals than very early masked states.

\textbf{Objective combination.}
Table~\ref{tab:ablation} (middle) shows that $\mathcal{L}_{\text{step}}$ alone outperforms $\mathcal{L}_{\mathrm{base}}$ alone, while their combination performs best, indicating thier complementary nature.

\textbf{Step-wise weight.}
Table~\ref{tab:ablation} (bottom) shows that all tested $\alpha_{\text{step}}\!\in\!\{0.1,0.5,0.9\}$ improve over $\alpha_{\text{step}}{=}0$, with the best result at $\alpha_{\text{step}}{=}0.5$.
Thus, the state-wise objective is consistently useful, but its weight should be balanced against the terminal objective.

\subsection{Robustness and Diagnostics of Intermediate-State Branching}
\label{sec:branching-diagnostics}
\begin{wraptable}{r}{0.43\linewidth}
\vspace{-1.0\baselineskip}
\centering
\small
\caption{\textbf{Block size of rollout trajectories at the training on Sudoku.}}
\label{tab:block-size}
\begin{tabular}{@{}l@{\,\,}c@{\,\,\,\,}c@{\,\,\,\,}c@{\,\,\,\,}c@{}}
\toprule
\textbf{Block size / $N_{\text{gen}}$} & \textbf{128} & \textbf{256} & \textbf{512} & \textbf{Avg.} \\
\midrule
(no training) & 11.7 & 6.7 & 5.5 & 8.0 \\
\midrule
$8$    & 29.5 & 26.9 & 26.3 & 27.6 \\
$32$ (default) & 34.7 & 30.2 & 29.4 & 31.4 \\
$128$ & 32.4 & 30.8 & 32.4 & 31.9 \\
\bottomrule
\end{tabular}
\vspace{-0.7\baselineskip}
\end{wraptable}
\textbf{Training rollout block size.}
\method\ is not tied to a particular rollout schedule: it only requires masked-position logits at selected intermediate states.
Table~\ref{tab:block-size} varies the block size used during \textit{training} rollouts and shows that \method\ remains effective beyond the default 32-token semi-autoregressive schedule.

\textbf{Intermediate states provide useful learning signals.}
Under 32-token block-wise rollouts, \textbf{17.8\%} of uniformly sampled intermediate states already yield different rewards across same-state branches, with the fraction increasing sharply later in denoising (Tab.~\ref{tab:diag-state} in Appx.~\ref{app:diag:inter}.). 
This confirms that intermediate states provide meaningful non-terminal supervision.

\textbf{Cached-branch stability.}
Cached logits are used only for branch sampling; optimization still uses current-vs.-old surrogate probabilities.
Replay experiments show that step-level $\log\rho$ remains comparable to terminal $\log\rho$, with very low clip fractions (Table~\ref{tab:diag-state}) in Appendix~\ref{app:diag:stability}), indicating no observable instability from cached-branch reuse.

\subsection{Compute Cost Breakdown}
\label{sec:exp:compute}

\begin{wraptable}{r}{0.44\linewidth}
\vspace{-1.2\baselineskip}
\centering
\small
\caption{\textbf{Additional operation counts per prompt with \method.}
$Z$: \# branches/state, $\mathcal{S}$: selected states, $N_m$: \# monte-carlo prompt masks for $\mathbb{E}_m[\cdot]$ in the surrogate.}
\label{tab:compute-budget}
\begin{tabular}{@{}lc@{}}
\toprule
\textbf{Compute items} & \textbf{$\Delta$ w/ \method} \\
\midrule
Diffusion rollout forward & $0$ \\
Optimizer update steps & $0$ \\
One-step surrogate (terminal) & $0$ \\
\midrule
Reward evaluations & \roundboxblue{$|\mathcal{S}|Z$} \\
One-step surrogate (step-wise) & \roundboxblue{$2N_m|\mathcal{S}|$} \\
\bottomrule
\end{tabular}
\vspace{-0.5\baselineskip}
\end{wraptable}
\textbf{Algorithmic overhead.}
Table~\ref{tab:compute-budget} summarizes the additional operations introduced by \method.
The dominant training costs---multi-step diffusion rollouts and optimizer update steps---are unchanged.
\method\ adds only lightweight extras: reward evaluations for same-state branches and \textit{one-step} surrogate calls for step-wise log-probabilities at selected states.

\textbf{Implementation-level wall-clock overhead.}
Although \method\ leaves the dominant algorithmic costs unchanged, our current implementation is, in practical, slower in wall-clock time ($\sim$0.4$\times$ baseline throughput at $Z{=}2$ and $|\mathcal{S}|{=}K$) due to non-optimized repeated policy/surrogate calls and kernel/IO.
This overhead is implementation-dependent rather than intrinsic to the rollout/update budget; likely optimizations are discussed in Appendix~\ref{app:wallclock}.
Figure~\ref{fig:wallclock} therefore compares wall-clock-matched training curves, where \method\ still catches up and surpasses the terminal-feedback baseline within the same time budget.

\begin{figure}[t]
    \centering
    \includegraphics[width=0.7\linewidth]{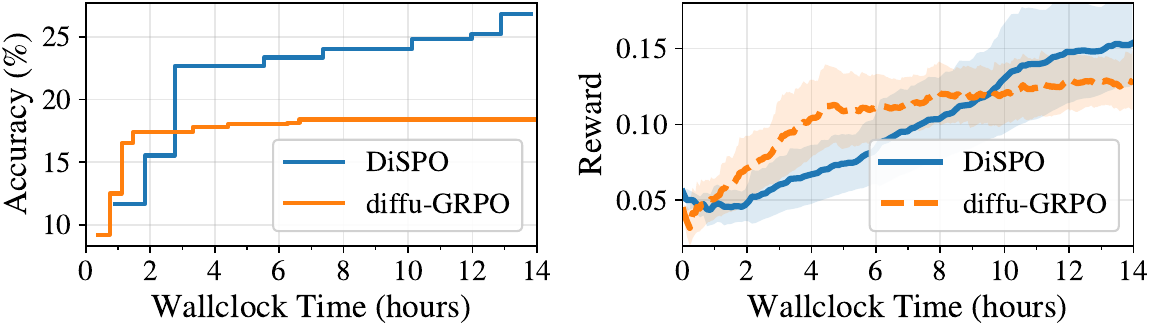}
\caption{\textbf{Wall-clock-matched training curves on LLaDA-8B-Instruct for Sudoku.} Accuracy ($N_\mathrm{gen}{=}128$) and reward vs.\ training time. \method\ surpasses diffu-GRPO within the budget.}
    \label{fig:wallclock}
\end{figure}

\subsection{Error Analysis: Premature Commitments}
\label{sec:qual}
\begin{wrapfigure}{r}{0.5\linewidth}
  \vspace{-1.4\baselineskip}
  \centering
  \includegraphics[width=\linewidth]{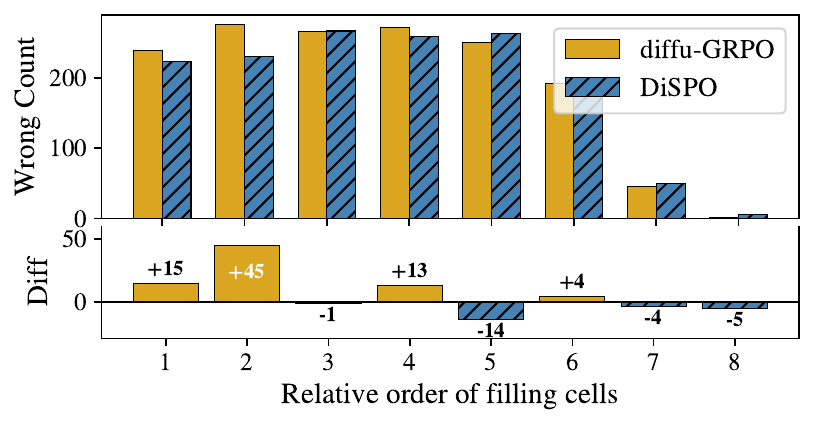}
  \caption{\textbf{Premature commitments on Sudoku.}
  }
  \label{fig:first-violation-hist}
  \vspace{-1.0\baselineskip}
\end{wrapfigure}
We analyze premature commitment errors on Sudoku using the final LLaDA-8B-Instruct checkpoint with $N_{\text{gen}}{=}128$.
We define the \textit{first-violation time} as the earliest denoising step at which a filled cell violates a Sudoku constraint.
Figure~\ref{fig:first-violation-hist} shows that \method\ shifts first violations to later denoising steps compared with diffu-GRPO, suggesting that optimization of intermediate actions reduces early, hard-to-correct constraint violations. Qualitative examples are provided in Figure~\ref{fig:qual-sudoku-case} in Appendix~\ref{app:qual-sudoku}.

\section{Related Work}
\label{sec:related} 
\textbf{Terminal-feedback policy optimization for MDLMs.}
Most RL fine-tuning methods for masked diffusion LMs optimize a scalar terminal reward on the final completion, using PPO/GRPO-style objectives with diffusion surrogate likelihoods~\citep{zhao2025d1,spg,tang2025wd1,cjgrpo,gdpo,sepo,diffpo}.
\method\ is complementary: it compares alternative mask fillings at fixed intermediate states via a state-wise objective.

\textbf{Terminal reward shaping.}
SAPO uses denoising states to shape terminal rewards with step-aware bonuses~\citep{sapo}.
This is complementary to \method: SAPO designs denser rewards along realized denoising processes, whereas \method\ optimizes filling actions under a fixed masked state.
We include SAPO as a non-matched-compute reference in Table~\ref{tab:main-results}, with compute details in Appendix~\ref{app:sapo}.

\textbf{Trajectory-level objectives.}
MDPO aligns the denoising process by comparing different denoising trajectories across steps and favoring trajectories whose intermediate predictions improve toward the terminal outcome~\citep{mdpo}.
In contrast, \method\ has the different optimization unit: it fixes an intermediate masked state and compares alternative fillings of the current masks.

\textbf{Surrogate likelihoods and gradient estimators.}
Complementary work improves surrogate likelihoods or policy-gradient estimators for diffusion LMs~\citep{spg,sepo,gdpo,d2}.
\method\ can directly incorporate improved surrogates/estimators while keeping the state-wise objective unchanged.

\section{Conclusion}\label{sec:conclusion}
We proposed \method, a plug-in objective that turns intermediate masked states into state-wise credit-assignment signals for MDLM policy optimization.
It reuses rollout-cached logits to compare alternative mask fillings with no additional multi-step diffusion rollouts.
On LLaDA-8B-Instruct, \method\ improves both diffu-GRPO and SPG across benchmarks under matched compute.

\textbf{Limitations and discussion.}
Future work includes extending \method\ beyond diffu-GRPO and SPG to objectives such as GDPO~\citep{gdpo}, stronger surrogate estimators, and reward shaping such as SAPO~\citep{sapo}.
Although \method\ adds no multi-step diffusion rollouts or optimizer steps, our implementation incurs wall-clock overhead from reward and surrogate evaluations.
Like other LM optimization methods, \method\ may improve reasoning and planning but could also strengthen misuse-capable models; our experiments are limited to public benchmarks and do not involve deployment or user data.


\clearpage
\section*{Acknowledgement}
These research results were obtained from the commissioned research (No.22501) by National Institute of Information and Communications Technology (NICT) , Japan.
This work was partially supported by JSPS KAKENHI Grant Number 25H01137 and JST K Program Japan Grant Number JPMJKP24C3.

\bibliographystyle{unsrtnat}
\bibliography{custom}

\appendix

\section{Theoretical Analysis}
\label{sec:theory}
\method\ treats intermediate diffusion states as decision points and directly optimizes their mask fillings via policy gradients.
Under the masked-token surrogate policy $\tilde{\pi}_\theta$ (Eq.~\ref{eq:subset-logprob-method}), 
we show that the step-wise loss provides a \emph{principled} policy-gradient estimator for a fixed-state objective (Theorem~\ref{thm:step-pg}).
We further show that combining step-wise and terminal losses corresponds in expectation to optimizing a single objective (Theorem~\ref{thm:mixed}).
We also provide simple variance-reduction mechanisms that justify token-local updates and same-state averaging (\S~\ref{sec:theory-var}).
For clarity, we analyze the \emph{unclipped} likelihood-ratio objective without KL.

\subsection{Step-wise Objective and Step-level Policy Gradient}
\label{sec:theory-step}

Fix a denoising step index $t$.
Let $d_t^{\mathrm{old}}(\cdot\mid q)$ denote the distribution over intermediate states at step $t$ induced by rolling out the frozen behavior model $\theta_{\mathrm{old}}$ on prompt $q$.
Using the state--action view in \S~\ref{sec:method-states}, define
\begin{equation}
  J_{\mathrm{step},t}(\theta)
  :=
  \mathbb{E}_{q \sim \mathcal{D}}
  \mathbb{E}_{s_t \sim d_t^{\mathrm{old}}(\cdot\mid q)}
  \mathbb{E}_{a \sim \tilde{\pi}_\theta(\cdot \mid s_t)}
  \bigl[ R\bigl(q,\textsc{Fill}(x_t,a)\bigr) \bigr],
  \label{eq:Jt-def}
\end{equation}
where $s_t=(q,x_t)$ and $\textsc{Fill}(x_t,a)$ deterministically completes $x_t$ by filling its masked positions with $a$.
Let $\mathcal{L}_{\text{step}}^{(t)}(\theta)$ denote the corresponding unclipped step-wise loss at timestep $t$, computed as in Eq.~\ref{eq:inter-loss-method} with $Z$ same-state branches.

\begin{theorem}[Step-level policy gradient]\label{thm:step-pg}
Assume (i) $d_t^{\mathrm{old}}(\cdot\mid q)$ is independent of $\theta$, and (ii) $\tilde{\pi}_\theta(\cdot\mid s)$ is differentiable and normalized for each $s$.
Using the within-group mean baseline in Eq.~\ref{eq:group-adv}, in the unclipped likelihood-ratio setting the expected gradient of the step-wise loss is aligned with the policy gradient of $J_{\mathrm{step},t}(\theta)$:
\begin{equation}
  \mathbb{E}\bigl[-\nabla_\theta \mathcal{L}_{\text{step}}^{(t)}(\theta)\bigr]
  =
  \frac{Z-1}{Z}\,\nabla_\theta J_{\mathrm{step},t}(\theta).
  \label{eq:step-pg-main}
\end{equation}
\end{theorem}

\emph{Proof idea.}
Write the step-wise loss (Eq.~\ref{eq:inter-loss-method}) as a score-function estimator for $\tilde{\pi}_\theta(a\mid s_t)$ using action-only log-probabilities on the current mask.
The likelihood-ratio form gives the standard importance-weighted policy gradient; the within-group mean baseline induces only the constant factor $(Z-1)/Z$ in expectation and reduces variance.
Full details are in Appendix~\ref{app:proof-step-pg}.

\subsection{Mixed Objective and Overall Policy Gradient}
\label{sec:theory-mixed}
A key property of \method\ is that adding step-wise updates does not introduce an ad-hoc signal: in the \emph{unclipped, exact likelihood-ratio} setting, the expected gradient of the combined loss (Eq.~\ref{eq:total-loss}) equals the gradient of a single well-defined objective.

\textbf{Terminal surrogate objective.} Define the terminal (sequence-level) surrogate objective \begin{equation} J_{\text{seq}}(\theta) := \mathbb{E}_{q \sim \mathcal{D},\, o \sim \tilde{\pi}_\theta(\cdot \mid q)} \bigl[ R(q,o) \bigr]. \label{eq:Jseq-def} \end{equation} A standard likelihood-ratio argument implies that the unclipped terminal loss $\mathcal{L}_{\mathrm{base}}(\theta)$ is an unbiased policy-gradient estimator for $J_{\text{seq}}(\theta)$~\citep{shao2024deepseekmath}.

\textbf{Mixed objective.}
Let $c_Z := (Z-1)/Z$ denote the constant scaling induced by the within-group mean baseline in Eq.~\ref{eq:group-adv}.
With timestep sampling $\omega(t)$ and the step-wise objectives $\{J_{\mathrm{step},t}(\theta)\}$ (Eq.~\ref{eq:Jt-def}), define
\begin{equation}
J_{\text{mix}}(\theta)
  :=
  \alpha_{\text{step}}\,c_Z \sum_t \omega(t)\, J_{\mathrm{step},t}(\theta)
  + \alpha_{\mathrm{base}}\, J_{\text{seq}}(\theta).
  \label{eq:Jmix-def}
\end{equation}

\begin{theorem}[Overall policy gradient]\label{thm:mixed}
Under the assumptions of Theorem~\ref{thm:step-pg} (with the within-group mean baseline) and the standard likelihood-ratio assumptions for $J_{\text{seq}}(\theta)$,
\begin{equation}
  \mathbb{E}\bigl[-\nabla_\theta \mathcal{L}(\theta)\bigr]
  =
  \nabla_\theta J_{\text{mix}}(\theta),
  \label{eq:mixed-main}
\end{equation}
where $\mathcal{L}(\theta)=\alpha_{\text{step}}\mathcal{L}_{\text{step}}(\theta)+\alpha_{\mathrm{base}}\mathcal{L}_{\mathrm{base}}(\theta)$.
\end{theorem}

\emph{Proof idea.}
Condition on a timestep $t$: Theorem~\ref{thm:step-pg} gives $\mathbb{E}[-\nabla_\theta \mathcal{L}_{\text{step}}^{(t)}(\theta)]=c_Z\,\nabla_\theta J_{\mathrm{step},t}(\theta)$.
Averaging over $t\sim\omega(t)$ yields the step term $\;c_Z\sum_t \omega(t)\nabla_\theta J_{\mathrm{step},t}(\theta)$.
Then add $\mathbb{E}[-\nabla_\theta \mathcal{L}_{\mathrm{base}}(\theta)]=\nabla_\theta J_{\text{seq}}(\theta)$ and use linearity of expectation.
See Appendix~\ref{app:proof-mixed} for details.

\subsection{Variance Reduction for the State-wise Estimator}
\label{sec:theory-var}
\method\ uses two simple mechanisms that reduce the variance of \method\'s state-wise gradient estimates: (i) restricting gradients to actionable (masked) tokens, and (ii) averaging over $Z$ candidates per state via cached-logit resampling.

\textbf{Token-local updates.}
Fix a state $s=(q,x)$ of length $L$ with masked set $M$ of size $m=|M|$, and let $R$ be the terminal reward of the resulting completion.
Let $g_i := \nabla_\theta \log \tilde{\pi}_\theta(a_i\mid s,i)$ denote a per-position score term.
Compare an estimator that propagates through all positions,
$\hat g_{\text{full}} := \sum_{i=1}^L g_i\,R$,
to the action-only estimator,
$\hat g_{\text{sub}} := \sum_{i\in M} g_i\,R$.

\begin{proposition}[Variance reduction by partial updates]
\label{prop:var}
Assume $\{g_i\}$ are independent, zero-mean with $\mathrm{Var}[g_i]=\sigma^2$, and $R$ is independent of $\{g_i\}$ with $\mathbb{E}[R^2]<\infty$.
Then
\[
  \mathrm{Var}[\hat g_{\text{sub}}]
  \le
  \frac{m}{L}\,\mathrm{Var}[\hat g_{\text{full}}].
\]
\end{proposition}

\emph{Proof idea.}
Under the assumptions, the variance of the sum scales linearly with the number of included positions, giving the factor $m/L$.
See Appendix~\ref{app:proof-partial}.

\textbf{Same-state averaging via cached-logit branching.}
At a fixed state $s$, \method\ can resample $Z$ independent branches and average their gradient contributions.
With any baseline $b(s)$ that is independent of each sampled action, the variance of the averaged estimator decreases at the standard $O(1/Z)$ rate.
Cached-logit branching avoids additional \emph{rollout} forward passes to obtain these $Z$ samples.

\begin{proposition}[Variance vs.\ number of branches]
\label{prop:compute}
Fix a state $s$ and draw i.i.d.\ actions $a_1,\dots,a_Z \sim \tilde{\pi}_\theta(\cdot\mid s)$ with rewards $R_z$.
Let $b(s)$ be any baseline independent of each $a_z$, and define the per-branch contribution
$\hat g_z := (R_z-b(s))\,\nabla_\theta \log \tilde{\pi}_\theta(a_z\mid s)$ with finite variance.
Let $\hat g := \frac{1}{Z}\sum_{z=1}^Z \hat g_z$ be the averaged estimator.
Then $\mathrm{Var}[\hat g] = O(1/Z)$.
\end{proposition}

\emph{Proof idea.}
The estimator is the average of $Z$ i.i.d.\ per-branch terms under an action-independent baseline, so its variance shrinks by $1/Z$.
See Appendix~\ref{app:proof-compute}.

\noindent\textbf{Remark (group-relative baseline used in practice).}
Our implementation uses the within-group mean baseline (Eq.~\ref{eq:group-adv}), which introduces dependence across branches but typically further reduces variance in practice; we empirically verify variance reduction from increasing $Z$ (\S~\ref{sec:experiments}).

\section{Proofs}\label{app:proofs}

\subsection{Proof of Theorem~\ref{thm:step-pg}}
\label{app:proof-step-pg}

We prove that the (unclipped) step-level loss yields a principled (scaled) policy-gradient estimator for the step-wise objective $J_{\mathrm{step},t}(\theta)$ when treating the surrogate $\tilde{\pi}_\theta$ as the policy.

\textbf{Setup.}
Fix a timestep index $t$. Recall the step-wise objective (Eq.~\ref{eq:Jt-def})
\[
  J_{\mathrm{step},t}(\theta)
  =
  \mathbb{E}_{q \sim \mathcal{D}}
  \mathbb{E}_{s_t \sim d_t^{\mathrm{old}}(\cdot\mid q)}
  \mathbb{E}_{a \sim \tilde{\pi}_\theta(\cdot \mid s_t)}
  \bigl[ R\bigl(q,\textsc{Fill}(x_t,a)\bigr) \bigr],
\]
where $s_t=(q,x_t)$ and $a$ fills the masked positions at $s_t$.
For brevity, define the terminal reward as a function of $(s_t,a)$:
\[
  \mathcal{R}(s_t,a) := R\bigl(q,\textsc{Fill}(x_t,a)\bigr).
\]

\textbf{Policy-gradient form of $\nabla_\theta J_{\mathrm{step},t}(\theta)$.}
Under Assumption~(i) of Theorem~\ref{thm:step-pg}, the state distribution
$d_t^{\mathrm{old}}(\cdot\mid q)$ is induced by the frozen behavior model $\theta_{\text{old}}$ and
is independent of $\theta$, so we may move the gradient through the outer
expectations:
\begin{align}
\nabla_\theta J_{\mathrm{step},t}(\theta)
&=
\mathbb{E}_{q \sim \mathcal{D}}
\mathbb{E}_{s_t \sim d_t^{\mathrm{old}}(\cdot\mid q)}
\left[
\nabla_\theta
\mathbb{E}_{a \sim \tilde{\pi}_\theta(\cdot \mid s_t)}
\bigl[\mathcal{R}(s_t,a)\bigr]
\right].
\label{eq:app-Jt-movegrad}
\end{align}
Under Assumption~(ii) (differentiable, normalized $\tilde{\pi}_\theta$), the
score-function identity gives
\[
\nabla_\theta
\mathbb{E}_{a \sim \tilde{\pi}_\theta(\cdot \mid s_t)}
\bigl[\mathcal{R}(s_t,a)\bigr]
=
\mathbb{E}_{a \sim \tilde{\pi}_\theta(\cdot \mid s_t)}
\bigl[\mathcal{R}(s_t,a)\,\nabla_\theta \log \tilde{\pi}_\theta(a \mid s_t)\bigr].
\]
Substituting into Eq.~\ref{eq:app-Jt-movegrad} yields
\begin{equation}
\nabla_\theta J_{\mathrm{step},t}(\theta)
=
\mathbb{E}_{q \sim \mathcal{D}}
\mathbb{E}_{s_t \sim d_t^{\mathrm{old}}(\cdot\mid q)}
\mathbb{E}_{a \sim \tilde{\pi}_\theta(\cdot \mid s_t)}
\bigl[\mathcal{R}(s_t,a)\,\nabla_\theta \log \tilde{\pi}_\theta(a \mid s_t)\bigr].
\label{eq:app-Jt-pg}
\end{equation}

\textbf{Expected gradient of the step-level loss.}
At a selected state $s_{k,t}$, DiSPO draws $Z$ branched action samples
$a_{k,t,1:Z}$ from the frozen behavior policy $\tilde{\pi}_{\theta_{\text{old}}}(\cdot\mid s_{k,t})$
and evaluates $\mathcal{R}(s_{k,t},a_{k,t,z})$ for each branch.
Ignoring clipping (as stated in the theorem), the per-state step loss is
\[
  \mathcal{L}_{\text{step}}^{(k,t)}(\theta)
  =
  -\frac{1}{Z}\sum_{z=1}^Z
  \rho_{k,t,z}(\theta)\,A_{k,t,z},
  \qquad
  \rho_{k,t,z}(\theta)
  :=
  \frac{\tilde{\pi}_\theta(a_{k,t,z}\mid s_{k,t})}{\tilde{\pi}_{\theta_{\text{old}}}(a_{k,t,z}\mid s_{k,t})},
\]
where $A_{k,t,z}$ is the group-relative advantage using the within-group mean baseline (Eq.~\ref{eq:group-adv}), i.e., $A_{k,t,z}=R_{k,t,z}-\bar R_{k,t}$ with $\bar R_{k,t}=\frac{1}{Z}\sum_{j=1}^Z R_{k,t,j}$.
Since $\tilde{\pi}_{\theta_{\text{old}}}$ is frozen, $\nabla_\theta \rho_{k,t,z}(\theta)
=
\rho_{k,t,z}(\theta)\,\nabla_\theta \log \tilde{\pi}_\theta(a_{k,t,z}\mid s_{k,t})$.
Therefore,
\begin{equation}
-\nabla_\theta \mathcal{L}_{\text{step}}^{(k,t)}(\theta)
=
\frac{1}{Z}\sum_{z=1}^Z
\rho_{k,t,z}(\theta)\,A_{k,t,z}\,\nabla_\theta \log \tilde{\pi}_\theta(a_{k,t,z}\mid s_{k,t}).
\label{eq:app-step-grad}
\end{equation}

Taking expectation over the branched samples $a_{k,t,1:Z}\sim \tilde{\pi}_{\theta_{\text{old}}}(\cdot\mid s_{k,t})$
and using importance weighting yields
\begin{align}
\mathbb{E}\bigl[-\nabla_\theta \mathcal{L}_{\text{step}}^{(k,t)}(\theta)\mid q,s_{k,t}\bigr]
&=
\mathbb{E}_{a_{1:Z} \sim \tilde{\pi}_\theta(\cdot\mid s_{k,t})}
\left[
\frac{1}{Z}\sum_{z=1}^Z (R_z-\bar R)\,\nabla_\theta \log \tilde{\pi}_\theta(a_z\mid s_{k,t})
\right],
\label{eq:app-step-exp1-new}
\end{align}
where $R_z=\mathcal{R}(s_{k,t},a_z)$ and $\bar R=\frac{1}{Z}\sum_{j=1}^Z R_j$.
Expanding the baseline term,
\[
\mathbb{E}\!\left[\frac{1}{Z}\sum_{z=1}^Z \bar R\,\nabla_\theta \log \tilde{\pi}_\theta(a_z\mid s_{k,t})\right]
=
\frac{1}{Z^2}\sum_{z=1}^Z\sum_{j=1}^Z
\mathbb{E}\!\left[R_j\,\nabla_\theta \log \tilde{\pi}_\theta(a_z\mid s_{k,t})\right].
\]
For $j\neq z$, $R_j$ is independent of $a_z$ and
$\mathbb{E}[\nabla_\theta \log \tilde{\pi}_\theta(a_z\mid s_{k,t})]=0$, so cross terms vanish.
Only the $j=z$ terms remain, yielding
\[
\mathbb{E}\!\left[\frac{1}{Z}\sum_{z=1}^Z \bar R\,\nabla_\theta \log \tilde{\pi}_\theta(a_z\mid s_{k,t})\right]
=
\frac{1}{Z}\,
\mathbb{E}_{a \sim \tilde{\pi}_\theta(\cdot\mid s_{k,t})}
\bigl[\mathcal{R}(s_{k,t},a)\,\nabla_\theta \log \tilde{\pi}_\theta(a\mid s_{k,t})\bigr].
\]
Therefore,
\[
\mathbb{E}\bigl[-\nabla_\theta \mathcal{L}_{\text{step}}^{(k,t)}(\theta)\mid q,s_{k,t}\bigr]
=
\frac{Z-1}{Z}\,
\mathbb{E}_{a \sim \tilde{\pi}_\theta(\cdot\mid s_{k,t})}
\bigl[\mathcal{R}(s_{k,t},a)\,\nabla_\theta \log \tilde{\pi}_\theta(a\mid s_{k,t})\bigr].
\]
Taking outer expectations over $q\sim\mathcal{D}$ and $s_{k,t}\sim d_t^{\mathrm{old}}(\cdot\mid q)$ yields
\[
\mathbb{E}\bigl[-\nabla_\theta \mathcal{L}_{\text{step}}^{(t)}(\theta)\bigr]
=
\frac{Z-1}{Z}\,\nabla_\theta J_{\mathrm{step},t}(\theta),
\]
which matches Theorem~\ref{thm:step-pg}. \qed

\textbf{Remark on the group baseline.}
If a leave-one-out baseline is used (independent of each branch action), the estimator is unbiased.
With the within-group mean baseline in Eq.~\ref{eq:group-adv} used in our implementation, the expected gradient is scaled by the constant factor $(Z-1)/Z$, which can be absorbed into $\alpha_{\text{step}}$ or the learning rate.

\subsection{Proof of Theorem~\ref{thm:mixed}}
\label{app:proof-mixed}

We prove that the expected gradient of the combined loss equals the gradient of the mixed objective.
Throughout, we treat the masked-token surrogate $\tilde{\pi}_\theta$ as the policy being optimized and work in the unclipped setting stated in \S~\ref{sec:theory-mixed}.

\textbf{Recall the mixed objective and combined loss.}
Let $c_Z := (Z-1)/Z$ denote the constant scaling induced by the within-group mean baseline in the step-wise estimator.
The mixed objective (Eq.~\ref{eq:Jmix-def}) is
\[
  J_{\text{mix}}(\theta)
  =
  \alpha_{\text{step}}\,c_Z\sum_t \omega(t)\,J_{\mathrm{step},t}(\theta)
  +
  \alpha_{\mathrm{base}}\,J_{\text{seq}}(\theta),
\]
where $J_{\mathrm{step},t}(\theta)$ is the step-wise objective (Eq.~\ref{eq:Jt-def}) and $J_{\text{seq}}(\theta)$ is the terminal surrogate objective (Eq.~\ref{eq:Jseq-def}).
The combined training loss is
\[
  \mathcal{L}(\theta)
  =
  \alpha_{\text{step}}\,\mathcal{L}_{\text{step}}(\theta)
  +
  \alpha_{\mathrm{base}}\,\mathcal{L}_{\mathrm{base}}(\theta).
\]
By linearity of differentiation,
\begin{equation}
  \nabla_\theta J_{\text{mix}}(\theta)
  =
  \alpha_{\text{step}}\,c_Z\sum_t \omega(t)\,\nabla_\theta J_{\mathrm{step},t}(\theta)
  +
  \alpha_{\mathrm{base}}\,\nabla_\theta J_{\text{seq}}(\theta).
  \label{eq:app-mixed-grad-decomp}
\end{equation}
It therefore suffices to show that
\(
\mathbb{E}\!\left[-\nabla_\theta \mathcal{L}_{\text{step}}(\theta)\right]
=
c_Z\sum_t \omega(t)\,\nabla_\theta J_{\mathrm{step},t}(\theta)
\)
and
\(
\mathbb{E}\!\left[-\nabla_\theta \mathcal{L}_{\mathrm{base}}(\theta)\right]
=
\nabla_\theta J_{\text{seq}}(\theta).
\)

\textbf{Step term.}
Recall that $\mathcal{L}_{\text{step}}(\theta)$ is computed by selecting timesteps according to $\omega(t)$ and applying the per-state step loss $\mathcal{L}_{\text{step}}^{(k,t)}(\theta)$ (Eq.~\ref{eq:inter-loss-method}) on the corresponding intermediate state(s).
Condition on a particular timestep $t$ being selected.
Under the assumptions of Theorem~\ref{thm:step-pg}, we have
\[
  \mathbb{E}\!\left[-\nabla_\theta \mathcal{L}_{\text{step}}(\theta)\,\middle|\, t\right]
  =
  c_Z\,\nabla_\theta J_{\mathrm{step},t}(\theta),
\]
where the expectation is over $q\sim\mathcal{D}$, $s_t\sim d_t^{\mathrm{old}}(\cdot\mid q)$, and the branched action samples at that state.
Taking expectation over $t\sim\omega(t)$ yields
\begin{equation}
\mathbb{E}\!\left[-\nabla_\theta \mathcal{L}_{\text{step}}(\theta)\right]
  =
  c_Z\sum_t \omega(t)\,\nabla_\theta J_{\mathrm{step},t}(\theta).
  \label{eq:app-step-term}
\end{equation}

\textbf{Terminal term.}
By definition,
\[
  J_{\text{seq}}(\theta)
  =
  \mathbb{E}_{q\sim\mathcal{D}}
  \mathbb{E}_{o\sim\tilde{\pi}_\theta(\cdot\mid q)}
  \bigl[R(q,o)\bigr].
\]
A standard score-function (REINFORCE) argument gives
\[
  \nabla_\theta J_{\text{seq}}(\theta)
  =
  \mathbb{E}_{q,o\sim\tilde{\pi}_\theta}
  \bigl[
    R(q,o)\,\nabla_\theta \log \tilde{\pi}_\theta(o\mid q)
  \bigr],
\]
and subtracting any baseline that depends on $q$ but not on $o$ preserves unbiasedness.
The terminal GRPO loss $\mathcal{L}_{\mathrm{base}}(\theta)$ is precisely the corresponding baseline-centered (and, in practice, importance-weighted) estimator; in the unclipped setting,
\begin{equation}
  \mathbb{E}\!\left[-\nabla_\theta \mathcal{L}_{\mathrm{base}}(\theta)\right]
  =
  \nabla_\theta J_{\text{seq}}(\theta),
  \label{eq:app-term-term}
\end{equation}
which is the standard PPO/GRPO identity (see, e.g.,~\citet{shao2024deepseekmath}).

\textbf{Combine the terms.}
Combining Eq.~\ref{eq:app-step-term} and Eq.~\ref{eq:app-term-term} and using linearity of expectation,
\[
\mathbb{E}\!\left[-\nabla_\theta \mathcal{L}(\theta)\right]
=
\alpha_{\text{step}}\,c_Z \sum_t \omega(t)\,\nabla_\theta J_{\mathrm{step},t}(\theta)
+
\alpha_{\mathrm{base}}\,\nabla_\theta J_{\text{seq}}(\theta)
=
\nabla_\theta J_{\text{mix}}(\theta),
\]
where the last equality follows from Eq.~\ref{eq:app-mixed-grad-decomp} with the definition of $J_{\text{mix}}(\theta)$ above.
This proves Theorem~\ref{thm:mixed}. \qed

\textbf{Remark (KL-regularized case).}
If the main text includes a KL penalty, the same argument applies after adding the term
$-\lambda\,\mathbb{E}_{q}[\mathrm{KL}(\tilde{\pi}_\theta(\cdot\mid q)\,\|\,\tilde{\pi}_{\text{ref}}(\cdot\mid q))]$
to $J_{\text{mix}}(\theta)$ and the corresponding penalty $\lambda\,\mathcal{L}_{\text{KL}}(\theta)$ to $\mathcal{L}(\theta)$; differentiation under the expectation yields the additional gradient term.

\subsection{Proof of Proposition~\ref{prop:var}}
\label{app:proof-partial}

We prove the variance reduction claim for the token-local (action-only) estimator under the stated independence assumptions.

\textbf{Setup.}
Fix a diffusion state $s=(q,x)$ of length $L$ with masked set $M\subseteq[L]$ of size $m=|M|$.
Let $R$ denote the (scalar) terminal reward associated with the filled candidate, and define per-position score terms
\[
g_i \;:=\; \nabla_\theta \log \tilde{\pi}_\theta(a_i \mid s,i),
\]
where $a_i$ is the token sampled at position $i$ (for $i\in M$) and $\tilde{\pi}_\theta(\cdot\mid s,i)$ is the one-step surrogate distribution at position $i$.
Consider the two estimators (as in \S~\ref{sec:theory-var})
\[
\hat g_{\text{full}} \;=\; \sum_{i=1}^L g_i\,R,
\qquad
\hat g_{\text{sub}} \;=\; \sum_{i\in M} g_i\,R.
\]

\textbf{Variance calculation.}
Under the assumptions of Proposition~\ref{prop:var}, $\{g_i\}_{i=1}^L$ are independent, zero-mean with $\mathrm{Var}[g_i]=\sigma^2$, and $R$ is independent of all $g_i$ with $\mathbb{E}[R^2]<\infty$.
First, independence across positions implies
\[
\mathrm{Var}\!\left[\sum_{i\in S} g_i R\right]
=
\sum_{i\in S} \mathrm{Var}[g_i R]
\quad\text{for any index set } S\subseteq[L].
\]
Second, since $g_i$ is independent of $R$ and $\mathbb{E}[g_i]=0$,
\[
\mathrm{Var}[g_i R]
=
\mathbb{E}[g_i^2 R^2] - (\mathbb{E}[g_i R])^2
=
\mathbb{E}[g_i^2]\mathbb{E}[R^2]
=
\sigma^2\,\mathbb{E}[R^2].
\]
Therefore,
\[
\mathrm{Var}[\hat g_{\text{full}}]
=
\sum_{i=1}^L \sigma^2\mathbb{E}[R^2]
=
L\,\sigma^2\mathbb{E}[R^2],
\qquad
\mathrm{Var}[\hat g_{\text{sub}}]
=
\sum_{i\in M} \sigma^2\mathbb{E}[R^2]
=
m\,\sigma^2\mathbb{E}[R^2].
\]
Taking the ratio yields
\[
\mathrm{Var}[\hat g_{\text{sub}}]
=
\frac{m}{L}\,\mathrm{Var}[\hat g_{\text{full}}],
\]
which implies the stated inequality in Proposition~\ref{prop:var} since $m\le L$.
\qed

\textbf{Remark (vector-valued gradients).}
If each $g_i$ is vector-valued, the same argument applies coordinate-wise.
Equivalently, writing $\mathrm{tr}\,\mathrm{Cov}[\cdot]$ for the trace of the covariance matrix, independence implies
$\mathrm{tr}\,\mathrm{Cov}[\hat g_{\text{sub}}] = \frac{m}{L}\,\mathrm{tr}\,\mathrm{Cov}[\hat g_{\text{full}}]$
under the same assumptions.

\subsection{Proof of Proposition~\ref{prop:compute}}
\label{app:proof-compute}

We prove the $O(1/Z)$ variance reduction from averaging $Z$ (approximately) independent branches at a fixed state.

\textbf{Setup.}
Fix a state $s$ and draw i.i.d.\ actions $a_1,\dots,a_Z \sim \tilde{\pi}_\theta(\cdot\mid s)$ with associated rewards $R_z$.
Let $b(s)$ be any baseline independent of each sampled action $a_z$, and define the per-branch contribution
\[
  \hat g_z \;:=\; (R_z-b(s))\,\nabla_\theta \log \tilde{\pi}_\theta(a_z\mid s).
\]
Define the averaged estimator
\[
  \hat g \;:=\; \frac{1}{Z}\sum_{z=1}^Z \hat g_z,
\]
and assume $\hat g_z$ have finite variance.

\textbf{Variance of an average.}
For the scalar case, independence gives
\[
\mathrm{Var}[\hat g]
=
\mathrm{Var}\!\left[\frac{1}{Z}\sum_{z=1}^Z \hat g_z\right]
=
\frac{1}{Z^2}\sum_{z=1}^Z \mathrm{Var}[\hat g_z]
=
\frac{1}{Z}\,\mathrm{Var}[\hat g_z],
\]
where $\mathrm{Var}[\hat g_z]$ is the same for all $z$ since $\{\hat g_z\}_{z=1}^Z$ are i.i.d.
Thus $\mathrm{Var}[\hat g]=O(1/Z)$.

For the vector-valued case, the same argument holds for covariance matrices. Under independence,
\[
\mathrm{Cov}[\hat g]
=
\mathrm{Cov}\!\left[\frac{1}{Z}\sum_{z=1}^Z \hat g_z\right]
=
\frac{1}{Z^2}\sum_{z=1}^Z \mathrm{Cov}[\hat g_z]
=
\frac{1}{Z}\,\mathrm{Cov}[\hat g_z],
\]
where the last equality uses that $\{\hat g_z\}_{z=1}^Z$ are i.i.d.\ and hence have identical covariance.
Therefore $\mathrm{tr}\,\mathrm{Cov}[\hat g] = \frac{1}{Z}\,\mathrm{tr}\,\mathrm{Cov}[\hat g_z] = O(1/Z)$.
\qed

\textbf{Remark (group-relative baseline used in practice).}
Our implementation uses the within-group mean baseline (Eq.~\ref{eq:group-adv}), which introduces dependence across branches.
Proposition~\ref{prop:compute} formalizes the standard $1/Z$ reduction for action-independent baselines; empirically, increasing $Z$ reduces variance in our setting as well (\S~\ref{sec:experiments}).

\section{Reward Curves}\label{reward-curve}
Figure~\ref{fig:qualitative-2x2} shows that terminal reward curves (\emph{top}) and step reward curves (\emph{bottom}) on LLaDA-8B-Instruct during policy optimization at the experiments (\S~\ref{sec:experiments}). 
Across tasks, \method\ reaches higher terminal rewards earlier and maintains them over training. Step rewards exhibit relatively smaller magnitudes but follow trends as terminal rewards, indicating their role as a complementary training signal.

\begin{figure*}[h]
  \centering
  \includegraphics[width=\textwidth]{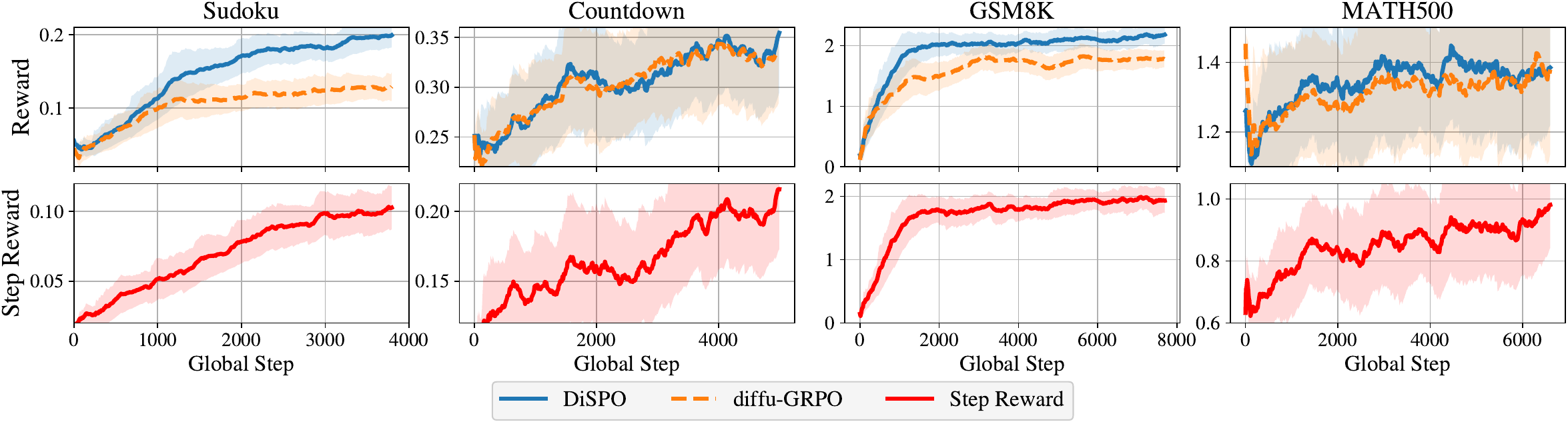}
  \caption{\textbf{Reward curves.} }
  \label{fig:qualitative-2x2}
\end{figure*}

\section{Experimental Details}
\label{app:exp}
This appendix provides additional details on benchmarks, reward functions, evaluation protocols, and training/implementation settings.
We follow the d1/diffu-GRPO setup where applicable~\citep{zhao2025d1}, reusing its SFT procedure and programmatic reward evaluators.
For policy optimization, we instantiate the terminal-feedback base objective with diffu-GRPO and SPG, and compare each baseline against its \method-augmented variant under matched compute-relevant quantities.
Run-specific hyperparameters, optimizer schedules, seeds, and infrastructure settings will be provided in the released configuration files.

\subsection{Benchmarks, Evaluation Protocol, and Reward Functions}
\label{app:exp:benchmark}
We evaluate on four reasoning benchmarks following \citet{zhao2025d1}.
For each task, we report exact-match accuracy and use a programmatic verifier to compute the scalar terminal reward $R(q,o)$ from prompt $q$ and model completion $o$.
The reward functions are task-specific and reused from the d1/diffu-GRPO implementation without modification: \url{https://github.com/dllm-reasoning/d1}.
Importantly, \method\ uses the same terminal reward evaluator for both the sequence-level/base objective and the state-wise objective; we do not introduce additional reward shaping or custom step-wise rewards.

\textbf{GSM8K.}
GSM8K~\citep{gsm8k} is a grade-school math word-problem benchmark.
It uses exact match on the final numeric answer after applying the same answer-extraction and normalization rules.
We use the training data publicly available \url{https://huggingface.co/datasets/openai/gsm8k}.
\emph{Reward} is computed by considering multiple axes, i.e., format reward (max. $1.625$) and correctness reward (max. $2.0$).

\textbf{MATH500.}
MATH500~\citep{lightman2023lets} is a subset of MATH focusing on competition-level problems.
\emph{Reward} is computed by considering the two axes, i.e., format reward (max $1.0$) and correctness reward (max $2.0$)

\textbf{Sudoku.}
4$\times$4 Sudoku tasks is synthetic benchmark for planning.
We use the training data publicly available \url{https://github.com/Black-Phoenix/4x4-Sudoku-Dataset}.
As for the evaluation data, we use the synthetically created problem set by using specific implementation \url{https://www.ocf.berkeley.edu/~arel/sudoku/main.html}.
Reward is the fraction of empty cells filled correctly.

\textbf{Countdown.}
Countdown is an arithmetic-planning benchmark, whose training data publicly available \url{https://huggingface.co/datasets/Jiayi-Pan/Countdown-Tasks-3to4}.
\emph{Evaluation} parses the output into an arithmetic expression, verifies that it uses the provided numbers under the task rules, and checks that the expression evaluates exactly to the target value.
\emph{Reward} is computed by considering the two axes, i.e., nearly-correct answers obtain $0.1$ and the perfect answers gain $1.0$.

\subsection{Inference Details}
\label{app:exp:inference}
For inference, we use semi-autoregressive decoding with 32-token blocks~\cite{arriolablock}, unmasking the 2 highest-confidence tokens per step within each block.
We evaluate at $N_{\text{gen}}\in\{128,256,512\}$ using zero-shot prompting and greedy decoding.

\subsection{Training and Implementation Details}
\label{app:exp:setting}
\textbf{Model variants and SFT.}
We evaluate two model variants: LLaDA-8B-Instruct~\citep{nie2025largelanguagediffusionmodels} and its s1k-supervised fine-tuned variant~\citep{s1k}.
For the s1k variant, we reuse the SFT recipe from d1, including the data, schedule, and hyperparameters: \url{https://github.com/dllm-reasoning/d1}.
We keep the architecture, tokenizer, and denoising schedule fixed unless explicitly varied in diagnostic experiments.

\textbf{Base optimizers and matched budgets.}
We instantiate the terminal-feedback base objective with diffu-GRPO~\citep{zhao2025d1} and SPG~\citep{spg}.
For diffu-GRPO, we follow the d1 training recipe.
For SPG, we implement the SPG objective in the d1 codebase so that data loading, reward evaluation, denoising, and training infrastructure remain shared.
For each base optimizer and task, we train the baseline and its \method-augmented variant with matched compute-relevant quantities, including $N_{\text{gen}}{=}256$, rollout budget, sequence-level group size, denoising schedule, and number of RL update steps.
The diffu-GRPO update steps are $3800/5000/7700/6600$ for Sudoku/Countdown/GSM8K/MATH500, respectively.
The SPG update steps are $4000/4600/4600/2600$ for Sudoku/Countdown/GSM8K/MATH500, respectively.
Run-specific settings are provided in the released configuration files.

\textbf{Other settings.}
Unless stated otherwise, \method\ uses lightweight branching settings: branch size $Z{=}2$ per selected intermediate state and one selected timestep per trajectory, $|T_{\text{sub}}|{=}1$.
Timesteps are sampled from a late-biased polynomial distribution $\omega(t)$ with exponent $k{=}4$.
We set $\alpha_{\mathrm{base}}{=}1$ and sweep the step-wise weight over $\alpha_{\text{step}}\in\{0.1,0.5,0.9\}$; run-specific choices are recorded in the released configuration files.

\textbf{SPG-specific settings.}
We implement SPG in the d1 codebase using the practical settings described in \citet{spg}.
For Monte Carlo estimation, we use the block-wise masking strategy: each sample randomly selects one block, keeps previous blocks mostly clean (with optional light perturbation on prompt/clean tokens), fully masks later blocks, and applies random masking within the selected block.
For negative-advantage traces, we use the mixture estimator from SPG, blending the EUBO and ELBO terms with coefficient $w$ when $A<0$.
We keep these SPG-specific choices fixed for both the SPG baseline and \method$_{\mathrm{SPG}}$.

\textbf{Optimization.}
Following \citet{zhao2025d1},
we use LoRA with a rank $r=128$ and scaling factor $\alpha = 64$, and AdamW optimizer~\cite{adamw}.
We basically set hyperparameters as follows: $\beta_1=0.9, \beta_2=0.99$, weight decay $0.1$, learning rate $3\times$ or $5 \times 10^{-6}$, gradient clipping $0.2$. 
We train on $4$ NVIDIA H100-94G GPUs with batch size 6 per GPU, gradient accumulation steps of 4, and sequence length of 256 tokens.

\section{Empirical Measurement of Step-level Gradient Variance}
\label{app:gradvar}

We detail the protocol used to produce Fig.~\ref{fig:gradvar}, which measures the variance of the step-level gradient estimator via repeated same-state branching.

\textbf{Quantity measured.}
Fix an intermediate diffusion state $s=(q,x)$ and draft size $Z$.
Let $\mathcal{L}_{\text{step}}(s;\theta)$ denote the \emph{unclipped} per-state step loss obtained by instantiating Eq.~\ref{eq:inter-loss-method} at state $s$ with $Z$ same-state branches, using the state-wise surrogate log-probability (Eq.~\ref{eq:subset-logprob-method}) and group advantages (Eq.~\ref{eq:group-adv}).
We define the step-level gradient estimator as the update direction
\[
\hat g_{\text{step}}(s)\;:=\;-\nabla_\theta \mathcal{L}_{\text{step}}(s;\theta)\in\mathbb{R}^d,
\]
We summarize the variance of this random vector by the trace covariance
\[
\mathrm{trCov}\bigl(\hat g_{\text{step}}(s)\bigr)
:=\mathrm{tr}\!\left(\mathrm{Cov}\bigl[\hat g_{\text{step}}(s)\bigr]\right)
=\mathbb{E}\!\left[\|\hat g_{\text{step}}(s)\|^2\right]-\left\|\mathbb{E}\!\left[\hat g_{\text{step}}(s)\right]\right\|^2.
\]
Empirically, for each fixed state $s$ we repeat same-state branching $R$ times to obtain $\{\hat g^{(r)}_{\text{step}}(s)\}_{r=1}^R$ and estimate
\[
\widehat{\mathrm{trCov}}\bigl(\hat g_{\text{step}}(s)\bigr)
=
\frac{1}{R-1}\sum_{r=1}^R \left\|\hat g^{(r)}_{\text{step}}(s)-\bar g_{\text{step}}(s)\right\|^2,
\qquad
\bar g_{\text{step}}(s)=\frac{1}{R}\sum_{r=1}^R \hat g^{(r)}_{\text{step}}(s).
\]

\textbf{Task, checkpoint, and state collection.}
We run this diagnostic on 4$\times$4 Sudoku for efficiency and use a fixed model checkpoint from DiSPO training (LLaDA optimized with DiSPO; step-3800).
We collect intermediate states from 64 prompts and sample 8 states per prompt, yielding 512 candidate states.

\textbf{State filtering and paired aggregation.}
To ensure a fair comparison across different conditions (e.g., different $Z$ and action-only vs full-token updates), we fix the same state set for all conditions and retain only non-degenerate states;
at the state (i) at least one actionable token contributes to the step loss (denominator $>0$), and
(ii) among the $R$ repeated branchings, at least one trial yields an advantage $>0$.
After filtering, we aggregate $\widehat{\mathrm{trCov}}(\hat g_{\text{step}}(s))$ over 276 states.

\textbf{Compared settings.}
$Z=2$ with action-only updates is the default setting used in our main experiments~(Table~\ref{tab:main-results}); 
Including that, for the $Z$ ablation, we test $Z\in\{2,3,4\}$; for the update-target ablation, we test the \emph{full-tokens} and \emph{action-only} variants.
Bars in Fig.~\ref{fig:gradvar} show the difference in gradient variance from the default ($Z=2$, action-only), with error bars indicating paired 95\% bootstrap CIs over states.

\section{Reference Comparison with Terminal Reward Shaping}\label{app:sapo}

\textbf{Difference in objectives.}
SAPO~\citep{sapo} and \method\ both leverage intermediate diffusion states, but they do so for different purposes.
SAPO scores intermediate states to \emph{shape the terminal rewards}, whereas \method\ directly optimizes \emph{intermediate-state actions}---the mask fillings sampled at a fixed partially-masked state---by comparing same-state counterfactual fillings under the unchanged terminal reward.
Hence, the two approaches are best viewed as complementary.

\textbf{Difference in compute.}
This difference in design also leads to substantially different compute profiles.
As summarized in Table~\ref{tab:compute-budget-sapo}, \method\ adds overhead mainly through extra reward evaluations for same-state branches and one-step surrogate evaluations for step-wise log-probabilities, while matching diffusion-rollout forward passes and update steps.
In contrast, SAPO typically incurs additional \emph{multi-step continuation rollouts} from intermediate states to estimate shaping terms, which increases the dominant cost---diffusion-rollout forward passes---often by a large margin.
Therefore, SAPO is not only complementary in approach, but can also be considerably more expensive in required rollout compute.

\begin{table}[h]
\centering
\caption{\textbf{Operation counts per prompt including SAPO.}
$K$: \# rollouts, $T$: \# denoising steps/rollout, $\mathcal{S}$: selected states (DiSPO), $Z$: \# branches/state,
$N_m$: \# MC prompt masks for $\mathbb{E}_m[\cdot]$, $U$: \# update steps.
For SAPO, $\mathcal{S}_{\mathrm{sapo}}$ is the set of intermediate states used for step-aware reward estimation,
$N_{\mathrm{sapo}}$ is the \# continuation rollouts per selected state, and $\bar T_{\mathrm{rem}}$ is the average remaining denoising steps.}
\label{tab:compute-budget-sapo}
\begin{tabular}{lccc}
\toprule
\textbf{Compute items} & \textbf{w/o \method} & \textbf{w/ \method} & \textbf{w/ SAPO} \\
\midrule
Diffusion rollout forward passes
& $KT$
& $KT$
& $KT + \roundboxorange{$N_{\mathrm{sapo}}|\mathcal{S}_{\mathrm{sapo}}|\bar T_{\mathrm{rem}}$}$ \\
Optimizer update steps
& $U$ & $U$ & $U$ \\
One-step surrogate (terminal logps)
& $2N_mK$ & $2N_mK$ & $2N_mK$ \\
\midrule
Reward evaluations
& $K$
& $K + \roundboxblue{$|\mathcal{S}|Z$}$
& $K + \roundboxorange{$N_{\mathrm{sapo}}|\mathcal{S}_{\mathrm{sapo}}|$}$ \\
One-step surrogate (step-wise logps)
& $0$
& \roundboxblue{$2N_m|\mathcal{S}|$}
& $0$ \\
\bottomrule
\end{tabular}
\end{table}

\textbf{Reference comparison in benchmarks.}
Despite the \emph{mismatch} in objectives and compute, SAPO is a representative step-aware method, so we include a \emph{reference} comparison to contextualize \method.
\textbf{\emph{This comparison should be interpreted with caution}}: the methods optimize different training signals (shaped vs.\ unshaped terminal rewards) and can require very different rollout compute (Table~\ref{tab:compute-budget-sapo}), so the evaluation is not an apples-to-apples matched-budget study.
Table~\ref{tab:main-results} reports that, on average, \method\ outperforms SAPO, despite using substantially less rollout compute.

\textbf{Takeaway.}
These results suggest that, among ways to leverage intermediate diffusion states, \emph{directly optimizing intermediate mask-filling actions} can be a particularly effective alternative to using intermediate-state scores only for reward shaping, achieving competitive performance \emph{at a fraction of the rollout compute}.
At the same time, since SAPO and \method\ act on different axes (reward shaping vs.\ action-level optimization), they are potentially complementary; we leave a careful study of joint training recipes to future work.

\section{Qualitative Examples of Premature Commitments}
\label{app:qual-sudoku}

Figure~\ref{fig:qual-sudoku-case} shows representative Sudoku completions from the final LLaDA-8B-Instruct checkpoint with $N_{\text{gen}}{=}128$.
The two methods are compared on the same instance at the same denoising step.
diffu-GRPO has already committed to a constraint-violating fill, whereas \method\ maintains a consistent partial assignment.

\begin{figure}[h]
  \centering
  \begin{subfigure}[b]{0.2\textwidth}
    \centering
    \includegraphics[width=\linewidth]{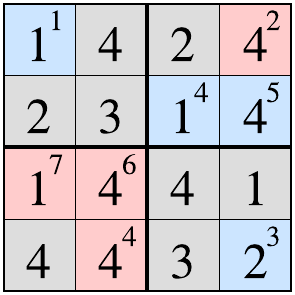}
    \caption{diffu-GRPO}
    \label{fig:qual-sudoku-grpo}
  \end{subfigure}
  \hspace{4mm}
  \begin{subfigure}[b]{0.2\textwidth}
    \centering
    \includegraphics[width=\linewidth]{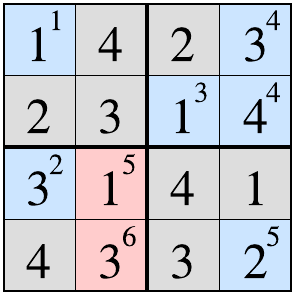}
    \caption{\method}
    \label{fig:qual-sudoku-dispo}
  \end{subfigure}

  \caption{\textbf{Qualitative example of premature commitment on Sudoku.}
  Gray cells are inputs; red and blue cells denote incorrect and correct fillings, respectively, with superscripts indicating fill order.
  diffu-GRPO makes an early constraint-violating fill, while \method\ keeps the partial assignment consistent at the same denoising step.}
  \label{fig:qual-sudoku-case}
\end{figure}

\section{Discussion on the Wall-clock Overhead}
\label{app:wallclock}

\textbf{Observed overhead.}
In our current implementation, \method\ runs at lower throughput than the terminal-feedback baseline (reported in \S~\ref{sec:exp:compute}) even at modest settings (e.g., $Z{=}2$).
This slowdown reflects engineering overhead rather than additional rollout/update compute, since the dominant algorithmic costs ($KT$ rollout forward passes and $U$ updates) are matched.

\textbf{Likely causes.}
The main contributors are repeated, non-amortized calls: (i) evaluating multiple same-state branches and (ii) computing step-wise log-probs via one-step surrogate calls at selected states.
When executed largely sequentially, these introduce frequent kernel launches, synchronization points, and host-side overhead that can dominate wall-clock time despite modest extra FLOPs.

\textbf{Optimization opportunities.}
These costs are not intrinsic to the method and should be reducible by (1) batching same-state branches (amortizing policy/surrogate/reward calls over $Z$), (2) caching state representations for step-wise terms, and (3) reducing kernel/IO overhead (e.g., call fusion / compilation and fewer synchronization points).
We therefore, as a reference, additionally reported wall-clock-matched curves (Fig.~\ref{fig:wallclock}) to ensure the performance gains (rewards and accuracy) are robust to current implementation overheads.

\textbf{Relation to inference-time decoding efficiency.}
This discussion concerns training-time optimization overhead.
It is orthogonal to inference-time MDLM acceleration methods, which reduce \emph{inference cost} by shortening denoising schedules, reusing/skipping computation across denoising steps, or early stopping~\citep{luxembourg2025plan,israel2025accelerating,wei2025accelerating,ma2025dkv,liu2025dllm,wu2025fast,oba2026stopping}.
In contrast, \method\ targets \emph{training-time} credit assignment by reusing rollout-cached logits to compare fixed-state mask fillings.

\section{Same-state resampling yields sparse but real signal}\label{app:diag:inter}
Our diagnostic, Table~\ref{tab:diag-state} shows that same-state resampling provides a nontrivial learning signal, although the signal density is highly non-uniform across timesteps. 
Under the conservative setting of uniform timestep sampling, \textbf{17.8\%} selected intermediate states already produce different rewards across same-state branches, and therefore provide usable supervision.

More importantly, this signal is highly concentrated in later denoising states. 
This helps explain the pattern in Table~\ref{tab:ablation} (Sec. 5.4): uniform sampling already provides a useful signal, while late-focused sampling performs better because it concentrates updates where reward-distinguishable alternatives are much more common.

\begin{table}[h]
\centering
\caption{\textbf{Fraction of reward-distinguishable intermediate states}; 64 prompts, 6 generations per prompt, 8 states per trajectory, 128 sampling steps. Timestep bins are defined over normalized denoising progress.}
\label{tab:diag-state}
\begin{tabular}{@{}lc@{}}
\toprule
\textbf{Timestep bin} & Reward-distinguishable states (\%) \\
\midrule
early $[0.0,0.3)$    & 4.5 \\
mid $[0.3,0.7)$      & 9.2 \\
late $[0.0,1.0]$     & 43.8 \\
\midrule
overall & 17.8 \\
\bottomrule
\end{tabular}
\end{table}

\section{Stability diagnosis of intermediate branching}\label{app:diag:stability}

For each checkpoint at global step $s\in\{900,1900,3800\}$, we resumed training and replayed
$\texttt{outer\_batches}=8$ with $\mu=12$ inner updates per outer batch (i.e., steps $s\!\rightarrow\!s+96$).
At every inner update, we computed branch-level $|\log\rho|$ separately for Terminal and Step losses:
Terminal uses completion-token likelihood ratios, while Step uses action-token likelihood ratios only.
Token-level log-ratios were aggregated within each branch by mean, then pooled over all inner updates.

Table~\ref{tab:ratio_stability_steps} reports pooled median, pooled p99, and clip fraction.
Sample sizes are $N_{\mathrm{term}}=9216$ for all rows, and $N_{\mathrm{step}}=15480$ ($900\!\rightarrow\!996$), $15584$ ($1900\!\rightarrow\!1996$), $15160$ ($3800\!\rightarrow\!3896$). 
For Terminal, the pooled sample size is $N_{\mathrm{term}}=\texttt{outer\_batches}\times\mu\times G\times B_{\text{dev}}\times A =8\times12\times4\times6\times4=9216$, where $G$ is the number of GPUs, $B_{\text{dev}}$ is per-device train batch size, and $A$ is gradient accumulation steps.
Equivalently, Step keeps one scalar $|\log\rho|$ per branch after token-mean aggregation over the
effective mask ($Z=2$).

\begin{table}[t]
\centering
\caption{\textbf{Stability diagnostics of cached-logit reuse on Sudoku.}}
\begin{tabular}{l ccc ccc}
\toprule
\multirow{2}{*}{Global Steps} & \multicolumn{3}{c}{Terminal} & \multicolumn{3}{c}{Step} \\
\cmidrule(lr){2-4}\cmidrule(lr){5-7}
& median & p99 & clip (\%) & median & p99 & clip (\%) \\
\midrule
900$\rightarrow$996   & 0.0065 & 0.0480 & 0.011 & 0.0053 & 0.0455 & 0.000 \\
1900$\rightarrow$1996 & 0.0071 & 0.0431 & 0.013 & 0.0046 & 0.0425 & 0.003 \\
3800$\rightarrow$3896 & 0.0065 & 0.0377 & 0.016 & 0.0048 & 0.0418 & 0.009 \\
\bottomrule
\end{tabular}
\label{tab:ratio_stability_steps}
\end{table}

\section{Existing Assets and Licenses}
\label{app:licenses}

Table~\ref{tab:asset-licenses} summarizes the existing assets used in our experiments.
We use these assets only for research evaluation and credit the original creators through citations and source links.
For assets whose license is not explicitly specified on the source page, we mark the license as ``not specified'' and refer readers to the original source.

\begin{table}[h]
\centering
\small
\caption{\textbf{Existing assets used in this work.}
Licenses are reported according to the source pages available at the time of writing.}
\label{tab:asset-licenses}
\begin{tabular}{@{}lll@{}}
\toprule
\textbf{Asset} & \textbf{Use} & \textbf{License / terms} \\
\midrule
LLaDA-8B-Instruct~\citep{nie2025largelanguagediffusionmodels} 
& Base model 
& MIT \\

d1 / diffu-GRPO code~\citep{zhao2025d1} 
& Training recipe, rewards, baselines 
& Apache-2.0 \\

s1K~\citep{s1k} 
& SFT data 
& Apache-2.0 \\

GSM8K~\citep{gsm8k} 
& Math benchmark 
& MIT \\

MATH500~\citep{lightman2023lets} 
& Math benchmark 
& MIT source split via PRM800K \\

Countdown~\citep{zhao2025d1} 
& Planning benchmark 
& Not specified on source page \\

4$\times$4 Sudoku dataset 
& Planning benchmark 
& Not specified on source page \\
\bottomrule
\end{tabular}
\end{table}

\end{document}